\documentclass[11pt]{article}
\usepackage{format}
\usepackage{macros}
\usepackage{oi-macros}

\title{
Outcome Indistinguishability
}
\author{Cynthia Dwork\thanks{This work supported in part by the Radcliffe Institute for Advanced Study at Harvard, Microsoft Research, the Simons Foundation Collaboration on the Theory of Algorithmic Fairness, the Sloan Foundation Grant 2020-13941, and NSF CCF-1763665.}
\\Harvard University\\\texttt{dwork@seas.harvard.edu}
\and
Michael P.\ Kim\thanks{Supported by the Miller Institute for Basic Research in Science. Part of this work completed at Stanford University, supported by NSF Award IIS-1908774.}
\\UC Berkeley\\\texttt{mpkim@berkeley.edu}
\and
Omer Reingold\thanks{Research supported by the Simons Foundation Collaboration on the Theory of Algorithmic Fairness, the Sloan Foundation Grant 2020-13941, Microsoft Research, and by NSF Award CCF-1763311.}
\\Stanford Unviersity\\\texttt{reingold@stanford.edu}
\and
Guy N.\ Rothblum\thanks{This project has received funding from the European Research Council (ERC) under the European Union’s Horizon 2020 research and innovation programme (grant agreement No. 819702), from the Israel Science Foundation (grant number 5219/17), from  the U.S.-Israel Binational Science Foundation  (grant 2018102), and from the Simons Foundation Collaboration on the Theory of Algorithmic Fairness. Part of this work was done while the author was visiting Microsoft Research.}

\\Weizmann Institute of Science\\\texttt{rothblum@alum.mit.edu}
\and
Gal Yona\thanks{This project has received funding from the European Research Council (ERC) under the European Union’s Horizon 2020 research and innovation programme (grant agreement No. 819702), from the Israel Science Foundation (grant number 5219/17) and from the Simons Foundation Collaboration on the Theory of Algorithmic Fairness. This research was also partially supported by the Israeli Council for Higher Education (CHE) via the Weizmann Data Science Research Center and by a research grant from Madame Olga Klein–Astrachan.}
\\Weizmann Institute of Science\\\texttt{gal.yona@weizmann.ac.il}}
\date{
\vspace{-30pt}
}

\begin{document}
\maketitle

\begin{abstract}
    
\newcommand{\comment}[1]{{\color{blue} #1}}
\newcommand{\rcomment}[2]{#1}
\newcommand{\acomment}[2]{#2}

Prediction algorithms assign numbers to individuals that are popularly understood as individual ``probabilities''---what is the probability of 5-year survival after cancer diagnosis?---and which increasingly form the basis for life-altering decisions.
Drawing on an understanding of computational indistinguishability developed in complexity theory and cryptography, we introduce {\it \OI}.
Predictors that are Outcome Indistinguishable 
yield a generative model for outcomes that cannot be efficiently refuted on the basis of the real-life observations produced by \nature.

We investigate a hierarchy of \OI definitions, whose stringency increases with the degree to which distinguishers may access the predictor in question.
Our findings reveal that \OI behaves qualitatively differently than previously studied notions of indistinguishability.
First, we provide constructions at all levels of the hierarchy. 
Then, leveraging recently-developed machinery for proving average-case fine-grained hardness,
we obtain lower bounds on the complexity of the more stringent forms of \OI.
This hardness result provides the first scientific grounds for
the political argument that, when inspecting algorithmic risk prediction instruments, auditors should be granted oracle access to the algorithm, not simply historical predictions.

 \end{abstract}

\pagenumbering{gobble}

\clearpage

\tableofcontents

\clearpage

\pagenumbering{arabic}

\section{Introduction}
\label{sec:intro}
Prediction algorithms ``score'' individuals, mapping them to numbers in $[0,1]$ that are popularly understood as ``probabilities'' or ``likelihoods:'' the probability of 5-year survival, the chance that the loan will be repaid on schedule, the likelihood, that the student will graduate within for years, or that it will rain tomorrow.  
Algorithmic risk predictions increasingly inform consequential decisions, but 
what can these numbers really mean? Five-year survival, four-year graduation, rain tomorrow, are not repeatable events.  
The question of ``individual probabilities'' has been studied for decades across many disciplines without clear resolution.\footnote{See the inspiring paper, and references therein, of Philip Dawid~\cite{Dawid}, discussing several notions of {\em individual risk} based on different philosophical understandings of probability ``including Classical,
Enumerative, Frequency, Formal, Metaphysical, Personal, Propensity, Chance and Logical conceptions of Probability'' and proposing a new approach to characterizing individual risk which, the author concludes, remains elusive.}

One interpretation relies on the coarseness of the representation of individuals to the prediction algorithm---the shape of a tumor's boundaries and the age of the patient; the student's grades, test scores, and a few bits about the family situation; the atmospheric pressure, humidity level, and winds---to partition individuals into a small number of ``types.''  
This leads to a natural interpretation of the predictions: \emph{amongst the individuals of this type, what fraction exhibit a positive outcome?}
In the context of modern data science, however, it is typical to make predictions based on a large number of expressive measurements for each individual---the patient's genome-wide risk factors; the borrower's online transactions and browsing data; the student's social media connections.
In this case, when each individual resolves to a unique set of covariates, the frequency-based interpretation fails.

Another view imagines the existence of a party---\nature---who selects, for each individual, a probability distribution over outcomes; then, the realized outcome is determined by a draw from this distribution.
Note that \nature may select the outcomes using complete determinism (i.e.,\ probabilities in $\set{0,1}$).
This view of the world gives rise to a statistical model with well-defined individual probabilities, but reasoning about these probabilities from observational data presents challenges.
Given only observations of outcomes, we cannot even determine whether \nature{} assigns integer or non-integer probabilities.  Perhaps \nature{} is deterministic, but we do not have enough information or computing resources to carry out the predictions ourselves.
Thus, even if we posit that outcomes are determined by individual probabilities, we cannot hope to recover the exact probabilities governing \nature , so this abstraction does not appear to provide an effective avenue for understanding the meanings of algorithmic risk scores.

\subsection{Predictions that Withstand Observational Falsifiability}

Given the philosophical uncertainty regarding the very existence of randomness, we explore the criteria for an ideal predictor.  We can view the outputs of
a prediction algorithm as defining a generative model for observational outcomes.
Ideally, the the outcomes from this generative model should ``look like'' the outcomes produced by \Nature .
To this end, we introduce and study a strong notion of faithfulness---\emph{\OI (OI)}.
A predictor satisfying \oi provides a generative model that cannot be efficiently refuted on the basis of the real-life observations produced by \nature.
In this sense, the probabilities defined by any OI \predictor provide a meaningful model of the ``probabilities'' assigned by \nature: even granted full access to the predictive model and historical outcomes from \nature, no analyst can invalidate the model's predictions.
Our study contributes a computational-theoretic perspective on the deeper discussion of what we should demand of prediction algorithms--a subject of intense study in the statistics community for over 30 years (see, {\it e.g.}, the forecasting work in~\cite{Dawid82,FosterVohra,Fudenberg,G3,sandroni2003reproducible})---and how they should be used.  For example, one of our results provides scientific teeth to the political argument that, if risk prediction instruments are to be used by the courts (as they often are in the United States), then at the very least oracle access to the algorithms should be granted for auditing purposes.

\OI presents a broad framework evaluating algorithmic risk predictions. 
This paper focuses on the fundamental setting of predicting a binary outcome, given an individual's covariates, a simple prediction setup that already highlights many of the challenges and subtleties that arise while defining and reasoning about OI.
Nothing precludes extending OI to reason about algorithms that make predictions about more general (\emph{e.g.},\ continuous) outcomes.

\paragraph{Basic notation.}
We assume that individuals are selected from some discrete domain $\X$, for example, the set of $d$-bit strings\footnote{Individuals can be arbitrarily complex; they are represented to the algorithm as elements of $\X$. Strictly speaking, distributions over $\X$ are induced distributions over the representations, and our results apply whether or not there are collisions.    
We do not assume that \nature 's view is restricted to the representation.}.  
We model \nature as a joint distribution, denoted $\Dnat$, over individuals and outcomes, where $\os_i \in \set{0,1}$ represents \nature's choice of outcome for individual $i \in X$.
We use $i \sim \DX$ to denote a sample from \nature's marginal distribution over individuals and denote by $\ps_i \in [0,1]$ the conditional probability that \nature assigns to the outcome $\os_i$, conditioned on $i$.
We emphasize, however, that \nature may choose $\ps_i \in \set{0,1}$ to be deterministic; our definitions and constructions are agnostic as to this point.

A {\em predictor} is a function $p:\X \to [0,1]$ that maps an individual $i \in \X$ to an estimate $\pt_i$ of the conditional probability of $\os_i = 1$.
For a predictor $\pt:\X \to [0,1]$, we denote by $(i,\ot_i) \sim \D(\pt)$ the  individual-outcome pair, where $i \sim \DX$ is sampled from \nature's distribution over individuals, and
then
the outcome $\ot_i \sim \Ber(\pt_i)$ is sampled
from the Bernoulli distribution with parameter $\pt_i$.

\paragraph{\OI.}
Imagine that \nature selects $\ps_i = 1$ for  half of the mass of $i\sim \DX$ and $\ps_i = 0$ for the remainder.
If the two sets of individuals are easy to identify then we can potentially recover a close approximation to $\ps$.
Suppose, however, that the sets are computationally indistinguishable, in the sense that given $i \sim \DX$, no efficient observer can guess if $\ps_i=1$ or $\ps_i=0$ with probability significantly better than $1/2$.
In this case, producing the estimates $\pt_i = 1/2$ \emph{for every individual} $i \in \X$ captures the best computationally feasible understanding of \nature:  given limited computational power, the outcomes produced by \nature may faithfully be modeled as a random.
In particular, if \nature were to change the outcome generation probabilities from $\ps$ to $\pt$ we, as computationally bounded observers, will not notice.
In other words, predictors satisfying OI give rise to models of \nature that cannot be falsified based only on observational data.

In the most basic form of the definition, a \predictor $\pt:\X \to [0,1]$ is Outcome Indistinguishable with respect to a family of distinguishers $\A$  if samples from \nature's distribution $\Dnatsample$ cannot be distinguished by $\A$ from samples from the \predictor's distribution $\Dptsample$, meaning that for each algorithm $A \in \A$, the probability that $A$ outputs $1$ is (nearly) the same on samples from each of the two distributions, $\Dnat$ and $\Dpt$.

\begin{definition*}[\OI]
Fix \nature's distribution $\Dnat$.
For a class of distinguishers\footnote{Like the predictors, distinguishers have access only to the finite {\em representations} of individuals as elements of $\X$.} $\A$ and $\eps > 0$,
a \predictor $\pt:\X \to [0,1]$ satisfies $(\A,\eps)$-\oi (OI) if for every $A \in \A$,
\begin{gather*}
\card{\Pr_{(i,\os_i) \sim \Dnat}\lr{A(i,\os_i; \pt) = 1}
- \Pr_{(i,\ot_i)\sim \Dpt}\lr{A(i,\ot_i; \pt) = 1}} \le \eps.
\end{gather*}
\end{definition*}
The definition of \OI can be extended in many ways, for example to settings where distinguishers receive multiple samples from each distribution, or when they have access to the program implementing $\pt$, and to the case of non-Boolean outcomes.

In the extreme, when we think of $\A$ as the set of all efficient computations, \oi sets a demanding standard for \predictors that model \nature.
Given an OI predictor $\pt$, even the most skeptical scientist---who, for example, does not believe that \nature can be captured by a simple computational model---cannot refute the model's predictions through observation alone.
This framing seems to give an elegant computational perspective on the scientific method, when consider $\pt$ as expressing a hypothesis that cannot be falsified through observational investigation.

\subsection{Our Contributions}
The most significant contributions of this work can be summarized as follows:
\begin{enumerate}[(1)]
    \item We define a practically-motivated four-level hierarchy of increasingly demanding notions of \emph{\OI}.
    The levels of the hierarchy arise by varying the degree to which the distinguishers may access the predictive model in question.  
\item We provide tight connections between the two lower levels of the hierarchy to {\em multi-accuracy} and {\em multi-calibration}, two notions defined and studied in~\cite{hkrr}.
    Establishing this connection immediately gives algorithmic constructions for these two levels.
    \item We describe a novel algorithm that
    constructs OI predictors directly.
This construction establishes an upper bound on the complexity of OI for the upper levels of the hierarchy (and, consequently, also allows us to recover the results of \cite{hkrr} through the OI framework).
    \item We show a {\em lower bound} for the upper levels of the hierarchy, demonstrating the tightness of our constructions. We prove that, under plausible complexity-theoretic assumptions, at the top two levels of the hierarchy the complexity of the predictors cannot be polynomial in the complexity of the distinguishers in $\A$ and in the distinguishing advantage $1/\eps$.
\end{enumerate}

Additionally, we revisit the
apparent
interchangeability of the terms ``test'' and ``distinguisher'' in the literature on pseudorandomness,
drawing a distinction
that is relevant to the  forecasting problem.
Our reults clarify the mathematical relationship between notions in the two literatures.

Next, we present a colloquial illustration of the different notions of the hierarchy.
While very natural, the notions within the hierarchy have not been fully considered in the literature on either forecasting or pseudorandomness.

\paragraph{The \OI Hierarchy.}
Imagine a medical board that wishes to audit the output of a program $\pt$ used to estimate the chances of five-year survival of patients under a given course of treatment.
We can view the medical board as a distinguisher $A \in \A$.
To perform the audit, the board receives historical files of patients and their five-year predicted ({\it i.e.}, drawn from $\Dpt$) or actual (drawn from $\Dnat$) outcomes.  The requirement is that these two cases be indistinguishable to the board.
\begin{enumerate}[(1)]
\item To start, the board is only given samples, and must distinguish \nature's samples $\Dnatsample$ from those sampled according to the predicted distribution $\Dptsample$.
The board gets no direct access to predictions $\pt_i$ of the program; we call this variant \emph{\one}.
\item Naturally, the board may ask to see the predictions $\pt_i$ for each sampled individual.
In this extension---\emph{\two}--- the board must distinguish samples of the form $(i,\os_i,\pt_i)$ and $(i,\ot_i,\pt_i)$, again for $\Dnatsample$ and $\Dptsample$.
\item \emph{\Three} allows the board to make queries to the program $\pt$ on arbitrary individuals, perhaps to examine how the algorithm behaves on related (but unsampled) patients.
\item Finally, in \emph{\four}, the board is allowed to examine not only the predictions from $\pt$ but also the actual code, {\it i.e.}, the full implementation details of the program computing $\pt$.
\end{enumerate}

On a different axis, we also consider multiple-sample variants of OI and show how these relate to the single-sample variants described above.  Multiple-sample OI is closer to the problem of online {\em forecasting} ({\it e.g.}, daily weather forecasting);  we explore connections between this variant of OI and the forecasting literature in Section~\ref{sec:related}.

\paragraph{The Lower Levels of the OI Hierarchy.}
We begin by examining the relationship between the different levels of the hierarchy.
We show that \one and \two are closely related to the notions of multi-accuracy and multi-calibration \cite{hkrr}, respectively, studied in the algorithmic fairness literature.
Very loosely, for a collection $\C$ of subpopulations of individuals, $(\C,\alpha)$-multi-calibration asks that a predictor $\pt$ be calibrated (up to $\alpha$ error\footnote{Defining approximate calibration is subtle; see a discussion in Section~\ref{sec:connections} with the formal definition of approximate multi-calibration.}) not just overall, but also when we restrict our attention to subpopulations $S \subseteq \X$ for every set $S\in\C$.
Here, calibration over $S$ means that if we restrict our attention to individuals $i\in S$ for which $\pt_i=v$, then the fraction individuals with positive outcomes (i.e., $i \in S$ such that $\os_i=1$) is roughly $v$.
We prove that \two with respect to a set of distinguishers $\A$ is ``equivalent" to $\C$-multi-calibration in the sense that each notion can enforce the other, for closely related classes $\C$ and~$\A$.
\begin{result}[Informal]
\label{result:mc-equiv}
For any class of distinguishers $\A$ and $\eps > 0$, there exists a (closely related) collection of subpopulations $\C_\A$ and $\alpha_\eps > 0$, such that $(\C_\A,\alpha_\eps)$-multi-calibration implies $(\A,\eps)$-\two.
Similarly, for any collection of subpopulations $\C$ and $\alpha > 0$, there exists a (closely related) class of distinguishers $\A_\C$ and $\eps_\alpha > 0$, such that $(\A_\C,\eps_\alpha)$-\two implies $\C$-multi-calibration.
\end{result}
Importantly, the relation between the class of distinguishers and collection of subpopulations preserves most natural measures of complexity; in other words, if we take $\A$ to be a class of efficient distinguishers, then evaluating set membership for the populations in $\C$ will be efficient (and vice versa).
\One is similarly equivalent to the weaker notion of multi-accuracy, which requires accurate expectations for each $S \in \C$, rather than calibration.

\paragraph{Feasibility and Constructions.}
We consider the question of whether efficient OI predictors always exist.  In particular, we ask,
\emph{Can we bound the complexity of OI predictors, independently of the complexity of \nature's distribution?}
The picture we uncover is subtle; we will see that \OI differs qualitatively from prior notions of indistinguishability.

First off, leveraging feasibility results for the fairness notions from \cite{hkrr}, we can obtain efficient predictors satisfying \one or \two, by reduction to \ma and \mc.
Informally, for each of these levels, we can obtain OI predictors whose complexity scales linearly in the complexity of $\A$ and inverse polynomially in the desired distinguishing advantage $\eps$.
The result is quite generic; for concreteness, we state the theorem using circuit size as the complexity measure.
\begin{result}
\label{result:one-two}
Let $\A$ be a class of distinguishers implemented by size-$s$ circuits.
For any $\Dnat$ and $\eps > 0$, there exists a predictor $\pt:\X \to [0,1]$ satisfying $(\A,\eps)$-\two (similarly, \one) implemented by a circuit of size $O(s/\eps^2)$.
\end{result}

Turning now to \three and \four predictors,
we obtain a general-purpose algorithm  for constructing OI predictors,
even when the distinguishers are allowed arbitrary access to the predictor in question.
This construction extends the learning algorithm for multi-calibration of \cite{hkrr} to the more general setting of OI.
When we allow such powerful distinguishers, the learned predictor $\pt$ is quantitatively less efficient than in the weaker notions of OI.
In this introduction, we state the bound informally, assuming the distinguishers are implemented by circuits with oracle gates (a formal treatment appears in Section~\ref{sec:beyond}).
As an example, if we let $\A$ be the set of oracle-circuits of some fixed polynomial size (in the dimension $d$ of individual's representations), and allow arbitrary oracle queries, then $\pt$ will be of size $d^{O(1/\eps^2)}$.
\begin{result}[Informal]
\label{result:three}
Let $\A$ be a class of oracle-circuit distinguishers implemented by size-$s$ circuits that make at most $q$ oracle calls to the predictor in question.
For any $\Dnat$ and $\eps > 0$, there exists a predictor $\pt:\X \to [0,1]$ satisfying $(\A,\eps)$-\three implemented by a (non-oracle) circuit of size $s\cdot q^{O(1/\eps^2)}$.
\end{result}

Intuitively, \four can implement any of the prior levels through simulation: given the code for $\pt$, the distinguishers can execute oracle calls (or calls to $\pt_i$) whenever needed.\footnote{There is some subtlety in making this intuition formal, related to encoding and decoding the description of the predictor, but the overhead is mild; see Section~\ref{sec:beyond}.}
At the extreme of efficient OI, we consider \four with respect to the class of polynomial-sized distinguishers.
Importantly, we allow the complexity of these distinguishers to grow as a (fixed) polynomial in both the dimension of individuals $d$ and the length of the description of the predictor $\pt$, which we denote by $n$.
For this most general version of OI, the complexity may scale doubly exponentially in $\poly(1/\eps)$; nevertheless, the bound is independent of the complexity of $\ps$. 
\begin{result}[Informal]
\label{result:four}
For some $d \in \N$, let $\X \subseteq \set{0,1}^d$ be represented by $d$-bit strings.
Suppose for some $k \in \N$, $\A$ is a class of distinguishers implemented by circuits of size $(d+n)^k$, on inputs $i \in \X$ and descriptions of predictors in $\set{0,1}^n$.
For any $\Dnat$ and $\eps > 0$, there exists a predictor $\pt:\X \to [0,1]$ satisfying $(\A,\eps)$-\four implemented by a circuit of size $d^{2^{O(1/\eps^2)}}$.
\end{result}

\paragraph{Hardness via Fine-Grained Complexity.}
We establish a connection between the fine-grained complexity of well-studied problems and the complexity of achieving \three. Under the assumption that the (randomized) complexity of counting $k$-cliques in $n$-vertex graphs is $n^{\Omega(k)}$, we demonstrate that the construction 
of Theorem \ref{result:three} is optimal up to polynomial factors. Specifically, we rule out (under this assumption) the possibility that the complexity of a \three predictor can be a fixed polynomial in the complexity of the distinguishers in $\A$ and in the distinguishing advantage $\eps$. This hardness result holds for constant distinguishing advantage $\eps$ and for an efficiently-sampleable distribution $\Dnat$. This hardness results are in stark contrast to the state of affairs for \two (see Theorem \ref{result:one-two}). Concretely, in the parameters of the upper bound, the result based on the hardness of clique-counting rules out any predictor $\pt$ satisfying \three of (uniform) size significantly smaller than $d^{\Omega(1/\epsilon)}$.  

\begin{result}[Informal]
\label{result:lb}

For $k \in \N$, assume there exist $\alpha > 0$ s.t. there is no $o(n^{\alpha \cdot k})$-time randomized algorithm for counting $k$-cliques. Then, there exist: $\X \subseteq \set{0,1}^{d^2}$, an efficiently-sampleable distribution $\Dnat$, and a class $\A$ of distinguishers that run in time $\tilde{O}(d^3)$ and make $\tilde{O}(d)$ queries, s.t. for $\eps = \frac{1}{100k}$, no predictor $\pt$ that runs in time $( d^{\alpha \cdot k} \cdot \log^{-\omega(1)}(d))$ can satisfy $(\A,\eps)$-\three.
\end{result}

See Corollary \ref{cor:clique-to-OI-hardness} for a formal statement. This lower bound is robust to the computational model: assuming that clique-counting requires $n^{\Omega(k)}$-sized circuits implies a similar lower bound on the circuit size of \three predictors. 
The complexity of clique counting has been widely studied and related to other problems in the fine-grained and parameterized complexity literatures, see the discussion in Section \ref{subsec:cliques-OI}. We note that, under the plausible assumption that the fine-grained complexity of known clique counting algorithms is tight, our construction shows that obtaining \three is as hard, up to sub-polynomial factors, as computing $\ps$. We emphasize that this is the case even though the running time of the distinguishers can be arbitrarily small compared to the running time of $\ps$.

\paragraph{Hardness via $\BPP \neq \PSPACE$.}
We also show that, under the (milder) assumption that $\BPP \neq \PSPACE$, there exists a polynomial collection of distinguishers and a distribution $\Dnat$, for which no polynomial-time predictor $\pt$ can be OI. The distinction from the fine-grained
result (beyond the difference in the assumptions) is that here $\Dnat$ is not efficiently sampleable, and the distinguishing advantage for which OI is hard is much smaller.

\begin{result}[Informal]
\label{result:lb-PSPACE}
Assume that $\BPP \neq \PSPACE$. Then, there exist: $\X \subseteq \set{0,1}^{d}$, a distribution $\Dnat$ (which can be sampled in $\exp(\poly(d))$ time), and a class $\A$ of $\poly(d)$ distinguishers that run in time $\poly(d)$, s.t. for $\eps = \frac{1}{\poly(d)}$, no predictor $\pt$ that runs in time $\poly(d)$ can satisfy $(\A,\eps)$-\three.
\end{result}

\paragraph{Discussion.}
We highlight a few possible interpretations and insights that stem from our technical results.
The ability to construct predictors that satisfy \oi can be viewed both positively and negatively.
On the one hand, the feasibility results demonstrate the possibility of learning generative models of observed phenomena that withstand very powerful scrutiny, even given the complete description of the model.
On the other hand, OI does not guarantee statistical closeness to \nature (it need not be the case that $\ps \approx \pt$). Thus, the  feasibility results demonstrate the ability to learn an {\em incorrect} model that cannot be detected by efficient inspection.

More generally, the computational perspective of OI underscores an inherent limitation in trying to recover the exact laws governing \nature from observational data alone.
We illustrate this perspective through a comparison to pseudorandomness.
Traditionally in pseudorandomness, our object of desire is \emph{random} (e.g.,\ a large string of random bits fed to a $\BPP$ algorithm), and we show that a simple \emph{deterministic} object suffices to ``fool'' efficient observers.
In \oi, our object of desire is a model of \nature, which may obey highly-complex \emph{deterministic} laws.
In this work, we show that a simple \emph{random} model of \nature---namely, $\Dpt$ for an OI predictor $\pt$---suffices to ``fool'' efficient observers.
In this sense, attempting to recover the ``true'' model of \nature based on real-world observations is futile:  no efficient analyst can falsify the outcomes of the random model defined by $\pt$, agnostic to the ``true'' laws of \nature.

The most surprising (and potentially-disturbing) aspect of our results may be the complexity of achieving  \three and \four.
In particular, for these levels, we show strong evidence that there exist $\ps$ and $\A$ that do not admit efficient OI \predictors $\pt$, \emph{even when $\A$ is a class of efficient distinguishers!}
That is, there are choices of \nature that cannot be modeled simply, even if all we care about is passing simple tests.
This stands in stark contrast to the existing literature on indistinguishability, where the complexity of the indistinguishable object is usually polynomial in the distinguishers' complexity and distinguishing advantage, regardless of the complexity of the object we are trying to imitate.

The increased distinguishing power of oracle access to the predictor in \three seems to have practical implications.
Currently, there are many conversations about the appropriate usage of algorithms when making high-stakes judgments about members of society, for instance in the context of the criminal justice system.
Much of the discussion revolves around the idea of \emph{auditing} the predictions, for accuracy and fairness.
The separation between \three and \two provides a rigorous foundation for the argument that auditors should at the very least have query access to the prediction algorithms they are auditing:  given a fixed computational bound, the auditors with oracle-access may perform significantly stronger tests than those who only receive sample access.

\subsection{Technical Overview}
\label{sec:overview}

Next, we give a technical overview of the main results.
Our goal is to convey the intuition for our findings, deferring the technical details to subsequent sections.

\paragraph{Relating OI and \mc.}
To build intuition for the equivalence, as described informally in Theorem~\ref{result:mc-equiv} and formally in Section~\ref{sec:connections}, we begin by describing the construction that establishes the lower level of the equivalence, between \ma and  \one (distinguishers that do not directly observe $\pt_i$). Informally, for a collection of subpopulations $\C$, \ma guarantees
that the expectations of $\ps_i$ and $\pt_i$ are approximately the same, even when conditioning on the event that $i \in S$ (simultaneously for every $S \in \C$). 

Given a subpopulation $S$, we define the \ma violation $\nabla_S(\pt)$ to be the absolute value of the above difference in conditional expectations. This can be viewed as a direct analogue of the distinguishing advantage $\Delta_A(\pt)$ (the absolute difference between the acceptance probability of $A$ on a sample from $\Dnat$ vs $\Dpt$).
To translate between the notions, we define two mappings (subpopulations to distinguishers, and distinguishers to subpopulations) that allow us to upper-bound one quantity in terms of the other. Specifically, given a collection of subpopulations $\C$, we define a family of distinguishers $\A = \set{A_S}$, where
for all $S \in \C$,  $A_S(i,b)=\textbf{1} \lr{i\in S \land b=1}$.

Note that the probability that each $A_S$ accepts on a sample $\Dptsample$ is the joint probability that $ i \in S$ and $\ot_i = 1$, which can be directly related to the  expectation of $\pt_i$ conditioned on $i \in S$. This allows us to  express the \ma violation in terms of the distinguishing advantage, as  
$\nabla_{S}(\pt)=\frac{\Delta_{A_{S}}(\pt)}{\Pr_{i\sim \DX}[i\in S]}$. By taking the accuracy parameter in \one to be sufficiently small, we can guarantee that \one implies \ma.  Similarly, to implement \oi from \ma, we define two sets $S_{(A,0)}$ and $S_{(A,1)}$ for every distinguisher $A \in \A$ and $b \in \set{0,1}$,
$S_{(A,b)}= \set{i \in \X:  A(i,b)=1}$.

Using similar arguments, we show that $\Delta_A(\pt)$ can be upper-bounded by $\nabla_{S_{(A,0)}}(\pt)+\nabla_{S_{(A,1)}}(\pt)$, for which \ma (w.r.t the constructed family of subpopulations) provides a bound. Note that the constructions are very closely related; in fact, repeating the translation twice reaches a ``stable point''. That is, given $\C$ (or given $\A$), we can construct a canonical pair $(\C', \A')$ such that $\C'$-\ma implies $\A'$-\one, and vice versa. Importantly, $\C'$ is (essentially) of the same complexity as $\C$ (resp., $\A'$ compared to $\A$), and the degradation in the accuracy parameters only results from the fact that \ma is defined for a collection of arbitrarily small sets.

Showing a similar equivalence for \mc follows the same general construction, but requires more care.
We begin with the observation that for a fixed predictor $\pt$, $\C$-\mc can be viewed as $\tilde{\C}$-\ma, where each subpopulation in $\tilde{\C}$ is obtained as the intersection of some subpopoulation $S \in \C$ and ``level-set'' of $\pt$: $\set{i \in S \land \pt_i = v}$.
Thus, at an intuitive level, 
we can model the constructions similarly in terms of the sets in $\tilde{\C}$.
A number of technical subtleties arise due to the precise notion of approximate calibration from \cite{hkrr}, which is necessary to provide sufficiently strong fairness guarantees; we discuss these issues in more detail in Section \ref{sec:connections}.

\paragraph{Constructing OI predictors.}
We establish the complexity of OI predictors (as in Theorems~\ref{result:three} and \ref{result:four}) by describing a learning algorithm that, given a class of distinguishers $\A$, an approximation parameter $\eps$, and samples from \nature's distribution $\Dnatsample$, constructs a predictor satisfying \oi, for any level of the OI hierarchy.
To start, inspired by the approach of \cite{hkrr}, we consider a reduction from the task of constructing an OI predictor to auditing for OI.
Specifically, the auditing problem receives a candidate predictor $p$, and must determine whether for all $A \in \A$, the distinguishing advantage $\Delta_A(p) < \eps$ is small; if there is an $A \in \A$ that has nontrivial advantage in distinguishing $\Dnat$ from $\Dmod{p}$, then the auditor must return such a distinguisher.
Naively, the auditor can be implemented by exhaustive search: for each $A \in \A$, the auditor---using the samples from $\Dnat$ as well as generated samples from $\Dmod{p}$---evaluates the advantage of $\Delta_A(p)$, returning $A$ if $\Delta_A(p) > \eps$.\footnote{For the sake of the overview of the construction, we ignore algorithmic issues of sample and time complexity in certifying that the estimate of each $\Delta_A(p)$ is sufficiently close.}

Suppose we're given some candidate predictor $p$; by iteratively auditing and updating, we aim to construct a circuit computing a predictor $\pt$ that satisfies OI.
To start, given the predictor $p$, if the auditor certifies that $p$ passes $(\A,\eps)$-OI, then trivially we have succeeded in our construction.
If, however, there is a distinguisher $A \in \A$ with a nontrivial advantage, then $p$ fails the audit; in this case, the successful distinguisher $A \in \A$ witnesses some ``direction'' along which $p$ fails to satisfy OI.
If we can update along this direction, a standard potential argument (akin to that of boosting or gradient descent) demonstrates that the updated predictor has made ``progress'' towards satisfying OI.
Then, we can recurse, calling the auditor on the updated predictor.
We argue, the process must terminate with an OI predictor after not too many rounds of auditing and updating. 

Thus, the crux of the construction is to solve the following problem:  given a circuit computing a predictor $p$ and a distinguisher $A \in \A$ that witnesses $\Delta_A(p) > \eps$, derive a new circuit computing an updated predictor $p'$ that has made progress towards OI.
A subtle issue arises when making this intuition rigorous for \three and \four.
At these levels of OI, the distinguishers may access the predictor in question, so there seems to be some circularity in the construction:  to obtain the OI predictor $\pt$, we need to call the distinguishers $A\in \A$; but to evaluate the distinguishers $A \in \A$, we may need to access $\pt$.
We argue that, in fact, there is no issue; to avoid the circularity, in each iteration, we can use the current predictor $p^{(t)}$ as the ``oracle'' for the distinguishers in $\A$.
If $p^{(t)}$ passes auditing by oracle distinguishers $A^{p^{(t)}}$, then this predictor satisfies \three.
If $p^{(t)}$ fails auditing, then we can still use the distinguisher  $A^{p^{(t)}}$ to derive an update that we argue makes progress towards \nature's predictor $\ps$.
Of course, because $\Dnat = \Dmod{\ps}$, $\ps$ satisfies OI.
Thus, the potential argument still works, and we guarantee termination after a bounded number of updates.\footnote{A similar argument holds for \four, using the description of $p^{(t)}$ as input.}

To finish the construction, we leverage the concrete assumptions about the model of distinguishers to build up the circuit computing $\pt$.
We focus on obtaining \three for size $s$ oracle circuits that make at most $q$ queries to $\pt$.
The argument above ensures that in the $t$-th iteration, we can implement each oracle distinguisher using (non-oracle) circuits, where each of the $q$ oracles calls is replaced with a copy of the current predictor $p^{(t)}$ hard-coded in place of the oracle gates.
Then, we can derive an updated circuit $p^{(t+1)}$ by combining the circuits computing $p^{(t)}$ and computing $A^{p^{(t)}}$ (taking an addition of the outputs, with appropriate scaling).
This recursive construction---where we build the circuit computing $p^{(t+1)}$ by incorporating multiple copies of $p^{(t)}$---suggests a recurrence relation characterizing an upper bound on the eventual circuit size.
Intuitively, with a base circuit size of $s$, and $q$ oracle calls (determining the branching factor), the size of $p^{(t)}$ grows roughly as $s \cdot q^t$.
Leveraging an upper bound on the number of iterations $T = O(1/\eps^2)$, the claimed bound follows.

Establishing the upper bound on the complexity of \four follows by a similar high-level argument, but there are some additional complications.
Briefly, because the distinguishers take, as input, the description of a circuit computing the predictor in question, we need to work with a class of distinguishers that \emph{grows with the complexity of the predictor itself}.
We deal with the technicalities needed to encode and decode predictors so that we can simulate the lower levels within \four.
We discuss these issues in full detail in Section~\ref{sec:beyond}.

\paragraph{Lower bounds for \three.}
To relate the complexity of evaluating \three predictors to complexity-theoretic assumptions such as the hardness of clique counting or $\PSPACE$-complete problems (as done in the informal results stated in Theorems \ref{result:lb} and \ref{result:lb-PSPACE}),
we consider the task of constructing an OI predictor that needs to withstand the scrutiny of a distinguisher that can make oracle queries.
Suppose we
guarantee that any such predictor must compute a moderately hard function $f$ on at least part of its domain.
Then, a distinguisher could use oracle access to the predictor (on that part of the domain) to perform expensive computations of $f$ at unit cost, while scrutinizing other parts of the domain.
As we'll see, pushing this intuition, we show how efficient oracle distinguishers can perform surprisingly powerful tests to distinguish $\pt$ from $\ps$.

With this intuition in mind, we divide the domain $\X$ into $m$ disjoint subsets $\X_1,\ldots,\X_m$.
As a first step, we want to make sure that it is moderately hard to achieve OI when the distribution $\Dnat$ is restricted to $\X_1$: for $i \in \X_1$, we set $\os_i = f_1(i)$, where $f_1$ is a moderately hard Boolean function.
We will guarantee that any \three predictor $\pt$ needs to compute $f_1$ on inputs in $\X_1$ by adding a distinguisher $A_1$ that verifies, for inputs $i \in \X_1$, that $\oo_i = f_1(i)$.
At this point, it isn't clear that anything interesting is happening: the complexity of achieving OI (i.e., computing $f_1$) is not yet higher than the complexity of the distinguisher (which also needs to compute $f_1$).
However, we can now add a distinguisher $A_2$ that verifies, for inputs in $\X_2$, that $\pt$ also correctly computes a harder function $f_2$.
The key point is for the distinguisher to verify that $\oo_i = f_2(i)$ {\em without computing $f_2$ itself}!
To achieve this, the distinguisher can use its oracle access to $\pt$. In particular, assuming that $\pt$ correctly computes $f_1$ on inputs $\X_1$, we can use a {\em downwards self-reduction} from computing $f_2$ on inputs in $\X_2$ to computing $f_1$ on inputs in $\X_1$.

The construction proceeds along these lines, using a sequence of functions $\{f_j\}_{j \in [m]}$, where for every $j \in [m]$ and $i \in \X_j$, we set $\os_i = f_j(i)$. Now, for every $j \in [2,\ldots,m]$, we want the function $f_j$ to be harder to compute than $f_{j-1}$, and we want a downwards self-reduction from computing $f_j$ to computing $f_{j-1}$. The distinguisher $A_j$ uses the given predictor $\pt$ as an oracle to $f_{j-1}$, and verifies that for $i \in \X_j$,  $\oo_i = f_j(i)$. We emphasize that the complexity of the oracle distinguisher $A_j$ is proportional to the cost of the downwards self-reduction, which can (and will) be significantly smaller than the complexity of computing $f_j$.

While intuitively appealing, the discussion above ignores an important point:  OI only provides an approximate guarantee on the real-valued predictions, not exact recovery of the sequence of fucntions $\set{f_j}$.
Starting at the first level of functions, an $(\{A_1\},\eps)$-\three predictor $\pt$ only has to correctly compute $f_1$ in a limited sense.
First, $\pt$ only needs to be correct {\em on average} for random inputs; it can err completely on some inputs.
Second, while we will choose $f_1$ (and all the $f_j$'s) to be a Boolean function, the predictor $\pt$ itself need not be Boolean. Nonetheless, for any input $i$, the distinguisher $A_1$ accepts the input $(i,\ot_i)$ only when $\ot_i = f_1(i)$, so taking $\eps$ small enough guarantees that for any $(\{A_1\},\eps)$-\three predictor $\pt$, with all but small probability over a draw of $i$ from $\DX$ restricted to $\X_1$, rounding $\pt(i)$ gives the correct value of $f_1$.
The probability of an error raises a new problem: the distinguisher $A_2$, which uses oracle calls to $\pt$ to compute $f_1$, might receive incorrect answers! Indeed, we expect the downwards self-reduction from $f_2$ to $f_1$ to make multiple queries (since $f_2$ is harder to compute), and so the error probability will be amplified.
To handle this difficulty, we also want a {\em worst-case to average-case reduction} for $f_1$: from computing $f_1$ on worst-case inputs in $\X_1$, to computing $f_1$ w.h.p. over random inputs drawn from $\DX$ restricted to $\X_1$.
Indeed, we'll want a similar reduction for each of the functions in the hierarchy.
For each $j \in [2,\ldots,k]$, the distinguisher $A_j$ will use the downwards self reduction, to $f_{j-1}$, using the worst-case to average-case reduction for $f_{j-1}$ to reduce the error probability of $\pt$ before answering the downward self-reduction's oracle queries. 

We can now make an inductive argument: assume that any $\pt$ that is OI for the distinguishers $\{A_1,\ldots,A_{j-1}\}$ must compute $f_{j-1}$ correctly w.h.p. over inputs drawn from $\DX$ restricted to $\X_{j-1}$.
Then $\pt$ is a very useful oracle for the $j$-th distinguisher $A_j$, which uses the downwards self-reduction and worst-case to average-case reductions, together with its oracle access to $\pt$, to compute $f_j$ (and verify that $\oo_i = f_j(i)$).
The key point is that $A_j$ can do this (via oracle access to $\pt$), even though its running time is much smaller than the time needed to compute $f_j$.
In turn, we conclude that any $\pt$ that is OI for the distinguishers $\{A_1,\ldots,A_{j}\}$ must compute $f_j$ correctly w.h.p. over inputs drawn from $\DX$ restricted to $\X_j$.
At the top ($m$-th) level of the induction, we conclude that a predictor that is OI for the entire collection of distinguishers must compute $f_m$ correctly w.h.p. over random inputs.
Finally, since we have a worst-case to average-case reduction for $f_m$, this implies that achieving OI is almost as hard as worst-case (randomized) computation of $f_m$.

To instantiate this framework, we need a collection of functions $\{f_j: \{0,1\}^n \rightarrow \{0,1\} \}_{j=1}^m$ with three properties:
(1) ``Scalable hardness'': The complexity of computing $f_j$ should increase with $j$. A natural goal is $n^{\Theta(j)}$ time complexity, where the lower bound should apply for randomized ($\BPP$) algorithms; (2) Downwards self-reduction: A reduction from computing $f_j$ to computing $f_{j-1}$, with fixed polynomial running time and query complexity (ideally $\tilde{O}(n)$, though we will use a collection where the complexity is a larger fixed polynomial); (3) Worst-case to average-case reduction: A reduction from computing $f_j$ in the worst case, to computing $f_j$ w.h.p. over a distribution $D_j$.

The clique counting problem presents a natural candidate, where $f_{j-2}$ counts the number of $j$-cliques in an unweighted input graph with $n$ vertices.\footnote{Clique counting is trivial for cliques of size 1 or 2, and begins being interesting for 3-cliques, or triangle counting.} The complexity of this well-studied problem is believed to be $n^{\Theta(j)}$.
Goldreich and Rothblum \cite{GoldreichR18} recently showed a worst-case to average-case reduction for this problem, where the reduction runs in $\tilde{O}(n^2)$ times and makes $\poly(\log n)$ queries.
The problem also has a downwards self-reduction from counting cliques of size $j$ to counting cliques of size $(j-1)$, which runs in time $O(n^3)$ and makes $n$ oracle queries (on inputs of size $O(n^2)$). 
The above construction utilized a sequence of Boolean functions, whereas the output of the clique-counting function is an integer in $[n^j]$. We use the Goldreich-Levin hardcore predicate \cite{goldreich-levin} to derive a Boolean function that is as hard to compute as clique counting.
The full details are in Section \ref{sec:lowerbound}.

The above framework can be also be instantiated using the algebraic variants of fine-grained complexity problems studied in the work of Ball, Rosen, Sabin, and Vasudevan \cite{ball2017average,BallRSV18}. Interestingly, downwards self-reducability also comes up in their work \cite{BallRSV18}, where it is used to argue hardness for batch evaluation of many instances. Their algebraic variants of the $k$-orthogonal-vectors and $k$-SUM problems seem directly suited to our construction. We focus on clique counting because of the tightness of the upper and lower bounds that have been suggested and studied in the literature.

For $\PSPACE$ hardness of \three, we use a $\PSPACE$-complete problem that is both downwards self-reducible and random self-reducible, due to Trevisan and Vadhan \cite{TrevisanV07}. The details are in Section \ref{subsec:PSPACE-OI}. 
The permanent \cite{Lipton89} (or scaled-down variants thereof, see \cite{GoldreichR18}) seems to be another promising candidate for our construction. In a different direction, it is interesting to ask whether moderately hard cryptographic assumptions, as suggested by Dwork and Naor \cite{DworkN92}, could also provide candidates.

\subsection{Broader Context and Related Notions}
\label{sec:related}

\paragraph{A Socio-Technical Path of Progress.}
A sufficiently rich representation of real-life individuals implies a mapping from individuals to their representation as input to the predictor that has no collisions.
In other words, given enough bits of information in the representation, each individual will have a unique representation.
Still, this richness does {\em not} mean that the representation contains the right information to determine the values of the $\ps_i$, even information-theoretically.
For example, modulo identical twins, individuals' genomes suffice to uniquely represent every person, but sequences of DNA are insufficient to determine an individual's ability to repay a loan.
Alternatively, the necessary information may be present, but its interpretation may be computationally infeasible.

Generally, we assume that the representation of individuals is fixed and informative.
Our analysis demonstrates that OI is feasible by using a potential argument.
Specifically, we describe an algorithm that iteratively looks for updates to the current set of predictions that will step closer towards indistinguishability. 
We guarantee that the algorithm terminates in a bounded number of steps by arguing that after sufficiently many updates, the predictor we hold is essentially $\ps$.

In practice, however, it may be the case that our features will be insufficiently rich to capture $\ps$.
Given a simple representation, even if we require a predictor $\pt$ that satisfies OI using a very computationally-powerful set of distinguishers ({\it e.g.},\ polynomial-sized circuits), there will be an inherent, information-theoretic limitation that prevents $\pt$ from converging to $\ps$.
While, given this representation of individuals, it may be impossible to distinguish \Nature's outputs $\os$ from $\ot$ drawn according to $\pt$, it may be possible to distinguish the outputs if we obtain an enriched representation of individuals.  Moreover, obtaining an enriched representation may even be easy, in that it can be accomplished (by a human) in time polynomial in the size of the original representation!

The OI framework can be extended naturally, allowing for the representation of individuals to be augmented throughout time.
Given such an enriched representation, we can continue iteratively updating $\pt$, based on the new representation.
Specifically, we can obtain new training data, concatenate the old and enriched representations to form a new representation, initialize a new predictor to equal $\pt$, and enrich the collection of distinguishers to operate on the new representation.
The new class of distinguishers retains and adds to the old distinguishing power, so $\pt$ likely will no longer satisfy OI; thus, we can apply our algorithm, starting at $\pt$, updating until the predictor fools the new class of distinguishers.
By applying the same potential argument, we can can guarantee that this process of augmenting the representations cannot happen too many times.
Any augmentation that significantly improves the distinguishing advantage between $\ps$ and $\pt$ must result in new updates that allow for significant progress towards $\ps$.

\paragraph{Prediction Indistinguishability}
In this work, we also investigate a notion we call {\em Prediction Indistinguishability (PI)}.
In PI we require the stronger condition that $(\ps,\os)$ be indistinghishable from $(\pt,\os)$.
While \PI is intuitively appealing, there are very simple distinguishers that show it is too much to ask for: a predictor $\pt$ is PI with respect to these distinguishers if and only if $\pt$ is statistically close to $\ps$.
But indistinguishability as a concept in computational complexity theory is interesting precisely when coming up with $\pt$ that is statistically close to $\ps$ is impossible.
Moreover, since we never see the values $\ps$---we don't even know if randomness exists!---we cannot hope to have indistinguishability of the $\pt$ from the ``true'' probabilities and it is strange even to pose such a criterion.  Nonetheless, we show that PI and OI are equivalent when indistinguishability is with respect to tests that are passed by all \naturelike histories with high probability (Section~\ref{sec:pi}), which we discuss in more detail next.

\paragraph{Tests vs.\ Distinguishers}
The framing of \oi draws directly from the notion of computational indistinguishability, studied extensively in the literature on cryptography, pseudorandomness, and complexity theory (see, e.g., \cite{oded-crypto,oded-vol2,salil-pseudo,oded-complexity} and references therein). A comparison to the extensive literature on online forecasting
clarifies
the semantic distinction between two concepts:  \emph{tests} (in the forecasting literature) and \emph{distinguishers} (in the complexity literature).

The forecasting literature focuses on an online setting where there are two players, \Nature{} and the Algorithm.
\Nature{} controls the data generating process (e.g.,\ the weather patterns), while the Algorithm tries to assess, on each Day~$t-1$, the probability of an event on Day~$t$ (e.g., will it rain tomorrow?).
There are many possibilities for \Nature; by definition, in this literature, \Nature{} calls the shots, in the following sense: On Day $t-1$, \Nature{} assigns a probability $\ps_t$ that governs whether it will rain or not on Day~$t$.
Note that \Nature{} is free to select $\ps_t \in \{0,1\}$, in which case the outcomes are determinsitic, $\os_t = \ps_t$

In the early 1980s, \cite{Dawid82} proposed that, at the very least, forecasts should be calibrated.  More stringent requirements were obtained by considering large (countable) numbers of sets of days, such as odd days, even days, prime-numbered days, days on which it has rained for exactly 40 preceding days and nights, and so on, and requiring calibration for each of these sets simultaneously~\cite{G3}.  This is the ``moral equivalent'' of multi-calibration in the world of infinite sequences of online forecasting.

A signal result in the forecasting literature, due to Sandroni~\cite{sandroni2003reproducible} applies to a more general set of tests than calibration tests.
Consider infinite histories, that is, sequences of (prediction, outcome) pairs.
We say a history is {\em \naturelike} if it is a sequence $\hist{p}{o}$ where $\forall t$ we have $o_t \sim \Ber(p_t)$.
Note that certain \naturelike{} histories may have no connection to any real-life weather phenomena, instead corresponding to a valid but unrealistic choice of $p$.
A {\em test} takes as input a (not necessarily \naturelike ) history and outputs a bit.  The test is usually thought of as trying to assess whether an algorithm's predictions are ``reasonably accurate'' with respect to the actual observations.
This implicitly focuses attention on tests that  \naturelike histories pass with high probability (over the draws from the Bernoulli distributions), and indeed, calibration tests fall into this category.
The goal of the Algorithm, then, is to generate predictions $\pt$ for which the histories $\hist{\pt}{\os}$ pass the test.  Here $\pt_i$ is the Algorithm's forecasted probability of rain for Day~$i$ and $\os_i$ is the Boolean outcome, rain or not, that actually occured on Day~$i$.

Sandroni's powerful result~\cite{sandroni2003reproducible}, proves, 
non-constructively\footnote{The result leverages Fan's minimax theorem.}, the existence of an Algorithm that, given any test $T$, yields a history which passes $T$ with probability at least as great as the minimum probability with which any \naturelike{} history $\nhist{p}$ passes~$T$ (again, the probability is over the draws from the $\Ber(p_i)$, $i \ge 1$).

There are two major differences between tests in the forecasting literature and distinguishers in the complexity-theoretic literature.\footnote{Unfortunately, and  confusingly, the literature on indistinguishability often conflates the notions, referring to distinguishers as tests.} The first is that tests have semantics---you want to pass the test, and the higher the probability of passing a test the better. In contrast, distinguishers output $0$ or $1$ with no semantics, and our goal is to produce an object such that the distinguisher outputs $1$ with the same probability as the objects that we are imitating.
In this case, getting the distinguisher to output $1$ with higher probability may be worse.
The second difference is that in the forecasting setting we want natural histories to pass the test with high probability: if natural histories fail the test, how can we interpret the Algorithm's inability to pass the test as an indication that the Algorithm is inaccurate?  As a result, the Algorithm does not compete with the actual \Nature $\ps$, but only with the hypothetical \nature that passes the test with the lowest probability.

To highlight these differences, consider a distinguisher that considers two sets of individuals, $S$ and $T$.
For each set, the distinguisher estimates the outcome probabilities $\alpha_S$ and $\alpha_T$---that is, averaged over the individuals $i \in S$ (respectively, $T$), the probability that $\os_i = 1$---and outputs $0$ with probability $\card{\alpha_S-\alpha_T}$ and $1$ otherwise.
Note that for some \Natures the distinguisher will output 1 with very small probability.
Nevertheless, in cases where \nature treats $S$ and $T$ equally, the distinguisher will output~$1$ with high probability and,
OI guarantees that $\pt$ must also treat $S$ and $T$ equally.
Many properties of samples from a distribution are quite naturally and directly specified through the language of distinguishers, and not obviously through the language of tests.
In light of this discussion, the connection between \two and multi-calibration is very interesting: it shows how to reduce a collection of distinguishers into a collection of tests, and even more specifically to a collection of calibration tests.

\paragraph{Algorithmic fairness.}
Tests are also implicit in the literature on algorithmic fairness, where they are sometimes referred to as {\em auditors}.
One line of work, the \emph{evidence-based fairness} framework---initially studied in \cite{hkrr,kgz,dwork2019learning}---relates directly to \oi and centers around tests that \nature always passes.
Broadly, the framework takes the perspective that, first and foremost, predictors should reflect the ``evidence'' at hand---typically specified through historical outcome data---as well as the statistical and computational resources allow.

Central to evidence-based fairness is the notion of \mc \cite{hkrr}, which was also studied in the context of rankings in \cite{dwork2019learning}. Recently,  \cite{jung2020moment} provide algorithms for achieving an extension of multi-calibration that ensures calibration of higher moments of a scoring function, and show how it can be used to provide credible prediction intervals. \cite{scm} study \mc from a sample-complexity perspective. In a similar vein, \cite{zhao2020individual} study a notion of individualized calibration and show it can be obtained by randomized forecasters.

Evidence-based fairness is part of a more general paradigm for defining fairness notions, sometimes referred to as ``multi-group'' notions, which has received considerable interest in recent years \cite{hkrr,knrw,krr,kgz,knrw2,dwork2019learning,scm,blum2019advancing,jung2020moment}.
This approach to fairness aims to strengthen the guarantees of notoriously-weak group fairness notions, while maintaining their practical appeal.
For instance, \cite{krr,knrw,knrw2} give notions of multi-group fairness based on parity notions studied in \cite{fta} and \cite{hardt2016equality}.
\cite{blum2019advancing} extend this idea to the online setting.
Other approaches to fairness adopt a different perspective, and intentionally audit for 
properties that \nature does not necessarily pass. Notable examples are group-based notions of parity \cite{hardt2016equality, kleinberg2016inherent, knrw, knrw2}.

\paragraph{Computational and Statistical Learning.}
Prediction tasks have also been studied extensively in the theoretical computer science and machine learning communities, both in the offline PAC model \cite{valiant1984theory}, as well as in the online model \cite{littlestone1994weighted,freund1997decision}; see \cite{shalev2014understanding} and references therein.
Relatedly, \cite{blum2007external} also show a multi-calibration-style guarantee in the online ``sleeping experts'' setting.
More broadly, our work is also in conversation with more applied approaches to learning distributions and generative models including GANs \cite{gans} or auto-encoders \cite{kingma2014autoencoding}.
The perspective of generating (statistically) indistinguishabile samples was also recently considered in a work introducing the problem of ``sample amplification'' \cite{axelrod2019sample}.

\paragraph{Organization.}
The remainder of the manuscript is structured as follows.
\begin{itemize}
\item Section~\ref{sec:oi} defines OI formally and shows a number of propositions relating the various notions of OI to one another.
\item Section~\ref{sec:pi} introduces the notion of prediction indistinguishability, and investigates the relationship of OI distinguishers to forecasting-style tests.
\item Section~\ref{sec:connections} demonstrates the connections between the first two levels of the OI hierarchy to \ma and \mc.
\item Section~\ref{sec:beyond} describes our construction of OI predictors, establishing the feasibility of the final two levels.
\item Finally, Section~\ref{sec:lowerbound} describes the construction establishing a conditional lower bound against the final levels.
\end{itemize}

\section{\OI}
\label{sec:oi}

Throughout this work, we study how to obtain \predictors that generate outcomes that cannot be distinguished from \naturelike outcomes.
Specifically, we model \nature as a joint distribution over individuals and outcomes, denoted $\Dnat$.
Individuals come from a discrete domain $\X$; throughout, we will assume that each $i \in \X$ can be resolved to some $d$-dimensional boolean string $i \in \set{0,1}^d$, which represents the ``features'' of the individual.
In this work, we focus on boolean outcomes $\Y = \set{0,1}$.
Thus, $\Dnat$ is supported on $\X \times \Y \subseteq \set{0,1}^d \times \set{0,1}$.
We use $i \sim \DX$ to denote a sample from \nature's marginal distribution over individuals.

We say that a \emph{\predictor} is a function $p:\X \to [0,1]$ that maps individuals to an estimate of the conditional probability of the individual's outcome being $1$.
For ease of notation, we use $p_i = p(i)$ to denote a \predictor's estimate for individual $i$.
Note that the marginal distribution over individuals $\DX$ paired with a \predictor induce a joint distribution over $\X \times \Y$.
Given a predictor $p$, we use $(i,o_i) \sim \D(p)$ to denote an individual-outcome pair, where $i \sim \DX$ is sampled from \nature's distribution over individuals, and the outcome  $o_i \sim \Ber(p_i)$ is sampled---conditional on $i$---according to the Bernoulli distribution with parameter $p_i$.

With this basic setup in place, we are ready to introduce the main notion of this work---\oi (OI).
Intuitively, when developing a prediction model, a natural goal would be to learn a \predictor $\pt:\X \to [0,1]$ whose outcomes ``look like'' \nature's distribution $\Dnat$.
\Oi formalizes this intuition, and is parameterized by a family of distinguisher algorithms $\A$.
In the most basic form of OI, each $A \in \A$ receives as input a labeled sample from one of two distributions,
\nature's distribution $\Dnat$ or the \predictor's distribution $\Dpt$.
\begin{align*}
\underbrace{\Dnatsample}_{\textrm{Nature's distribution}}&&\underbrace{\Dptsample}_{\textrm{\Predictor's distribution}}
\end{align*}
In other words, in each distribution the individual $i$ is sampled according to nature's marginal distribution on inputs, denoted $i \sim \DX$.
The distribution over outcomes, however, varies: conditioned on the individual $i$, the distinguisher either receives the corresponding \naturelike outcome $\os_i$, or receives an outcome sampled as $\ot_i \sim \Ber(\pt_i)$.
In its most basic form, a predictor $\pt$ satisfies OI over the family $\A$ if for all $A \in \A$, the probability that $A$ accepts the sample $(i,o_i)$ is (nearly) the same for \nature's distribution and the \predictor's distribution.
In addition to the sample from $\Dnat$ versus $\Dpt$, we can also allow the distinguishers to access the \predictor $\pt$ itself.
This setup allows us to define a prototype for a notion of OI.
\begin{definition}[\OI]
Fix \nature's distribution $\Dnat$.
For a class of distinguishers $\A$ and $\eps > 0$,
a \predictor $\pt:\X \to [0,1]$ satisfies $(\A,\eps)$-\oi (OI) if for every $A \in \A$,
\begin{gather*}
\card{\Pr_{(i,\os_i) \sim \Dnat}\lr{A(i,\os_i; \pt) = 1}
- \Pr_{(i,\ot_i)\sim \Dpt}\lr{A(i,\ot_i; \pt) = 1}} \le \eps.
\end{gather*}
\end{definition}
The subsequent sections introduce multiple variants of \oi, highlighting four distinct access patterns to $\pt$.
By allowing the distinguishers increasingly liberal access to the predicitve model $\pt$, the indistinguishability guarantee becomes increasingly strong.

\paragraph{Remark on nature.}
We primarily model \nature $\Dnat$ as a fixed and unknown joint distribution over $\X \times \Y$.
The presentation of some result benefits from an equivalent view, based on the agnostic PAC framework \cite{haussler1992decision,kearns1994efficient,kearns1994toward}.
In this view, we imagine that individuals $i \sim \D_\X$ are sampled from the marginal distribution over $\X$, and then \nature selects outcomes conditioned on $i$.
Throughout, we will use $\ps:\X \to [0,1]$ to denote the function that maps individuals to the true conditional probability of outcomes given the individual.
That is, for all $i \in \X$:
\begin{equation*}
\ps_i = \Pr_{\Dnatsample}\lr{\os_i = 1 \given i}.
\end{equation*}
In our notation, we can imagine that \nature specifies the distribution over individuals $i \sim \DX$, then specifies the ``\naturelike \predictor'' $\ps$ and samples the outcome for an individual $i$ as $\os_i \sim \Ber(\ps_i)$; in other words, $\Dnat = \Dmod{\ps}$.
We emphasize that this view---of \nature selecting a \predictor---is an abstraction that is sometimes instructive in our analysis of OI.
Nevertheless, we make no assumptions about $\ps$ other than it defines a valid conditional probability distribution for each $i \in \X$.
In particular, $\ps$ need not come from any realizable or learnable class of functions.

\paragraph{Expectations and norms.}
We often will take expectations over \nature's marginal distribution over individuals, possibly conditioned on membership in particular subpopulations $S \subseteq \X$.
A simple but important observation is that for any subpopulation $S \subseteq \X$, the expectation of the \naturelike outcome is equal to the expectation of $\ps$.
\begin{equation*}
    \Pr_{(i,\os_i) \sim \Dnat}\lr{\os_i = 1 \given i \in S} = \E_{(i,\os_i) \sim \Dnat}\lr{\os_i \given i \in S} = \E_{i \sim \DX}\lr{\ps_i \given i \in S}
\end{equation*}
Similarly, we may compare \predictors over the distribution of individuals.
For any two predictors $p,p':\X \to [0,1]$, we use the following $\ell_1$-distance to measure the statistical distance between their outcome distributions $\Dmod{p}$ and $\Dmod{p'}$.
\begin{gather*}
    \norm{p-p'}_1 = \E_{i \sim \DX}\lr{\card{p_i - p_i'}}
\end{gather*}
We only use the $\norm{\cdot}_1$ notation when the distribution on individuals $\DX$ is clear from context.

\paragraph{Supported predictions.}
In many definitions, we can reason about \predictors as arbitrary functions $p:\X \to [0,1]$, but to be an effective definition, we need to discuss functions that are implemented by a realizable model of computation.
Importantly, this means we will think of predictors as mapping individuals $i \in \X$ to a range of values $p_i$ that live on a discrete subset of $[0,1]$.
We assume for any predictor $p:\X \to [0,1]$, the \predictor's support is a discrete set of values in $[0,1]$ that receive positive probability over $\DX$.
For any subpopulation $S \subseteq \X$, we denote the support of $p$ on $S$ as
\begin{gather*}
\supp_S(p) = \set{v \in [0,1] : \Pr_{i \sim \DX}\lr{p_i = v \given i \in S} > 0}
\end{gather*}
In this way, for any $v \in \supp(p)$, the conditional distribution over individuals $i \sim \DX$ where we condition on the event $p_i = v$ is well-defined.

When possible, we obtain results agnostic to the exact choice of discretization.
Sometimes, we need to reason about the discretization explicitly and map the values of $\supp(p)$ onto a known grid with fixed precision; we introduce additional technical details as needed in the subsequent sections.

\paragraph{Distinguishers and subpopulations.}
The notion of \oi is parameterized by a family of distinguishing algorithms, which we denote as $\A$.
To instantiate a concrete notion of OI (at any of the four levels we define), we must specify $\A$ within a fixed realizable model of computation.
In practice, it may make sense to use a class of learning-theoretic distinguishers, (e.g.,\ decision-trees, halfspaces, neural networks).
In this work, we focus on more abstract models of distinguishers.
When our proofs allow, we will treat $\A$ as an arbitrary class of computations, but for certain results, it will be easier to assume something about the model of computation in which $A \in \A$ are implemented (e.g.,\ time-bounded uniform, size-bounded non-uniform).

Recall, we assume the domain of individuals $\X \subseteq \set{0,1}^d$ can be represented as $d$-dimensional boolean vectors for $d \in \N$, and that the distinguishing algorithms $A \in \A$ take as input an individual $i \in \X$ and an outcome $o_i \in \set{0,1}$.
Often, we will think of the dimension $d$ as fixed.
In this case, we can think of $\A$ as a fixed class of distinguishers of concrete complexity: for example, if the class $\A$ is implemented by circuits, then we can reason about their size as $s(d) = s$ for some fixed $s \in \N$.
When we think of the dimension as growing, then we need to consider ensembles of distinguishing families, where the family is parameterized by $d \in \N$.

The same issues arise when we discuss multi-calibration, which is parameterized by a collection of ``efficiently-identifiable'' subpopulations $\C \subseteq \set{0,1}^\X$.
Here, efficiently-identifiable refers to the fact that we assume for each $S \in \C$, there exists some efficient computational model that given $i \in \X$, can compute the predicate $\1\lr{i \in S}$ (i.e., the characteristic function of $S$).
Again, whenever possible, our treatment does not depend on the exact model of computation.

\paragraph{Circuits.}
As suggested above, some results are most naturally stated with a concrete model of computation in mind.
In these cases, we will describe the distinguishers and subpopulations as computed by a collection of circuits.
Fixing such a model of circuits will be useful when relating the complexity of a class of distinguishers $\A$ to that of a class of subpopulations $\C$, as well as showing the feasibility of obtaining circuits that implement OI predictors.
Analogous results could be proved instead for uniform classes.

Throughout, we say that a family of distinguishers $\A$ (resp., subpopulations $\C$) for $\X \subseteq \set{0,1}^d$ is implemented by a family of circuits of size $s(d)$, if for each $A \in \A$ (resp., $S \in \C$), there exists a bounded fan-in circuit over the complete boolean basis $c_A$ that computes the distinguisher $A$ on all inputs, with at most $s(d)$ gates (or equivalently, by bounded fan-in, $\Theta(s(d))$ wires).

Specifying the model of computation for the most stringent levels of OI requires some care.
The third level---\three---allows the distinguishers oracle-access to the \predictor in question.
For each $A \in \A$, we denote the oracle distinguisher as $A^\pt$, which has random access to $\pt_i$ for any $i \in \X$.
The fourth level---\four---allows the distinguishers direct access to the description of the \predictor in question, denoted $\ptstring$.
In this case, it makes sense to allow the ensemble of distinguishers to be parameterized by the length of the description $n = \card{\ptstring}$ in addition to the dimension $d$.
We discuss the specific assumptions about the implementation of these notions in subsequent sections.

\subsection{Defining the Levels of \OI}

With the general framework and preliminaries in place, we are ready to define the various levels of \oi.
In this section, we focus on the definitions of each notion---\one, \two, \three, and \four.
Along the way, we show some relations between the notions, but defer most of our investigation of the notions to subsequent sections.
We begin introducing each notion in the single-sample setting, and discuss OI for distinguishers that receive multiple samples in Section~\ref{sec:multi-sample}.

\subsubsection{\ONE}
The weakest model of distinguisher receives no direct access to the predictive model $\pt$, and must make its judgments based only on the observed sample.
In this framework, the only access to the \predictor is indirect, through the sampled outcomes.
\begin{definition}[\ONE]
Fix \nature's distribution $\Dnat$.
Let $\Z = \X \times \set{0,1}$.
For a class of distinguishers $\A \subseteq \set{\Z \to \set{0,1}}$ and $\eps > 0$,
a predictor $\pt:\X \to [0,1]$ is $(\A,\eps)$-\one if for every $A \in \A$,
\begin{gather*}
\card{\Pr_{\Dnatsample}\lr{A(i,\os_i) = 1}
- \Pr_{\Dptsample}\lr{A(i,\ot_i) = 1}} \le \eps.
\end{gather*}
\end{definition}

We remark that from a statistical perspective, \one already defines a strong framework for indistinguishability.
Even at this baseline level, if we allow a computationally-inefficient class of distinguishers, \one can be used to require closeness in statistical distance.
For instance, consider a family of ``subset'' distinguishers, where for a subset $S \subseteq \X$, the distinguisher $A_S$ is defined as follows.
\begin{gather*}
    A_S(i,o_i) = \begin{cases}
    o_i &\textrm{if }i \in S\\
    0 &\textrm{o.w.}
    \end{cases}
\end{gather*}
If we take $\A = \set{A_S : S \subseteq \X}$ to be the family of all subset distinguishers, then the only predictors $\pt$ that satisfy \one will be statistically close to \nature's \predictor $\ps$.
Of course, this class of distinguishers includes inefficient tests (necessary to certify $\norm{\ps - \pt}_1$ is small).
Our interest will be on the guarantees afforded by OI when we take $\A$ to be a class of efficient distinguishers.

\subsubsection{\TWO}
To strengthen the distinguishing power, we define \two, which allows distinguishers to observe the prediction for the individual in question.
Specifically, in addition to the sampled individual $i \sim \D$ and outcome $o_i$ (drawn according to nature or the \predictor $\pt$), the distinguisher receives the prediction $\pt_i$.
\begin{definition}[\TWO]
Fix \nature's distribution $\Dnat$.
Let $\Z = \X \times \set{0,1} \times [0,1]$.
For a class of distinguishers $\A \subseteq \set{\Z \to \set{0,1}}$ and $\eps > 0$, a predictor $\pt:\X \to [0,1]$ is $(\A,\eps)$-\two if for every $A \in \A$,
\begin{gather*}
\card{\Pr_{\Dnatsample}\lr{A(i,\os_i,\pt_i) = 1}
- \Pr_{\Dptsample}\lr{A(i,\ot_i,\pt_i) = 1}} \le \eps.
\end{gather*}
\end{definition}
\Two is a strengthening of \one: for any \one distinguisher, on input $(i,o_i,\pt_i)$, a \two distinguisher can simply ignore the prediction $\pt_i$, and simulate the original \one distinguisher.

\subsubsection{\THREE}

The next strengthening of OI allows distinguishers to make queries to $\pt$, not just on the sampled individual $i \sim \D$, but also on any other $j \in \X$.
Such a query model needs to be formalized; at a high level, we assume the distinguishers in the class $A \in \A$ are augmented with oracle access to $\pt$, denoted as $A^\pt$.
\begin{definition}[\THREE]
\label{def:three}
Fix \nature's distribution $\Dnat$.
Let $\Z = \X \times \set{0,1}$.
For a class of oracle distinguishers $\A \subseteq \set{\Z \to \set{0,1}}$ and $\eps > 0$, a predictor $\pt:\X \to [0,1]$ is $(\A,\eps)$-\three if for every $A^\pt \in \A$,
\begin{gather*}
\card{\Pr_{\Dnatsample}\lr{A^\pt(i,\os_i) = 1}
- \Pr_{\Dptsample}\lr{A^\pt(i,\ot_i) = 1}} \le \eps.
\end{gather*}
\end{definition}
The exact formulation of such oracle distinguishers will vary based on the model of computation in which $\A$ is defined.
In Section~\ref{sec:beyond}, we give one concrete formalization for distinguishers that are implemented efficiently as boolean circuits.
Independent of the exact model, \three can implement \two:  on input $(i,o_i)$, the \three distinguisher can access $\pt_i$ using a single query and then simulate the \two distinguisher.

\paragraph{Lunchtime-OI and \TWO.}
\Three generally defines a stronger notion of indistinguishability than \two, but we show that if the \three distinguishers are non-adaptive---asking only pre-processing queries---then they can be simulated by a family of (non-uniform) \two distinguishers.
This result demonstrates that the power of oracle-access distinguishers over sample-access distinguishers derives from the ability to query $\pt$ \emph{adaptively}, based on the sample in question.
In particular, in Section~\ref{sec:lowerbound} we show that \three is strictly more powerful than \two.
The construction follows by exploiting correlations within $\Dnat$ across different $i,j \in \X$, which can be tested efficiently by querying $\pt$ adaptively.

In fact, this collapse from \three to \two will hold for an even more powerful class of distinguishers, which are allowed ``lunchtime attack'' style pre-processing on $\pt$.
Consider the following model of pre-processing analysis.
For some $t \in \N$, given a family of distinguishers $\A$, suppose that for each $A \in \A$, there exists a pre-processing algorithm $R_A^\pt:1^d \to \set{0,1}^t$ with oracle access to $\pt$.
Given access to $\pt$ for input domain $\X \subseteq \set{0,1}^d$, the pre-processing algorithm $R_A^\pt(1^d)$ produces an advice string $a \in \set{0,1}^t$.
Then, oracle access to $\pt$ is revoked, and the distinguisher $A$ receives a individual-outcome-prediction sample $(i,o_i,\pt_i)$ from one of the two distributions, given access to $a$.\footnote{Note that, as in \two, we additionally give $\pt_i$ as input to the lunchtime distinguisher.
We exclude the prediction $\pt_i$ as input in Definition~\ref{def:three} because, in general, an adaptive \three distinguisher can query $\pt_i$ as desired.
Without feeding the prediction as input, lunchtime-OI actually collapses to \one.}
That is, the lunchtime variant of $(\A,\eps)$-\three holds if for every $A \in \A$ and $a = R_A^\pt(1^d)$,
\begin{gather*}
\card{\Pr_{\Dnatsample}\lr{A^a(i,\os_i,\pt_i) = 1}
- \Pr_{\Dptsample}\lr{A^a(i,\ot_i,\pt_i) = 1}} \le \eps.
\end{gather*}
Note that computing $R_A^\pt(1^d)$ need not be efficient, but importantly, its analysis of $\pt$ must be summarized into $t$ bits.
For this variant of \three, we show the following inclusion.
\begin{proposition}\label{prop:preprocess}
Fix \nature's distribution $\Dnat$.
Let $\Z = \X \times \set{0,1} \times [0,1]$.
Suppose $\A \subseteq \set{\Z \to \set{0,1}}$ is a class of lunchtime distinguishers implemented by size-$s$ circuits. 
Then, there exists a class of sample-access distinguishers $\A'$ implemented by size-$s$ circuits, such that any predictor $\pt:\X \to [0,1]$ that satisfies $(\A',\eps)$-\two must also satisfy $(\A,\eps)$-\three.
\end{proposition}
\begin{proof}
Given a class of lunchtime distinguishers $\A$, we define a new class $\A'$ of sample-access distinguishers as follows.
\begin{gather*}
    \A' = \set{A_a' : A \in \A,\ a \in \set{0,1}^t}
\end{gather*}
where $A_a'$ is defined as
\begin{gather*}
    A_a'(i,o_i,p_i) = A^a(i,o_i,p_i)
\end{gather*}
for all $i \in \X$ and $o_i \in \set{0,1}$ and $p_i \in [0,1]$.
In other words, for each $A \in \A$, we introduce $2^t$ fixed distinguishers that have the possible output of $R^\pt_A(1^d)$ hard-coded.
If $A$ is implemented by a circuit with access to the advice string output by $R^\pt_A(1^d)$, then for any $a \in \set{0,1}^t$, $A_a'$ can be implemented by circuits with the same number of wires.
We argue that $(\A',\eps)$-\two implies $(\A,\eps)$-\three.

Suppose there is some $A \in \A$ such that $A^a$ distinguishes between the natural and modeled distribution.
Then, by construction, there exists some $A_a' \in \A'$ that also distinguishes the distributions with the same advantage.
Thus, by contrapositive, if a predictor $\pt$ satisfies $(\A',\eps)$-\two, then it also satisfies $(\A,\eps)$-\three.
\end{proof}
Note that we state and prove Proposition~\ref{prop:preprocess} for distinguishers implemented by circuits, but the construction is quite generic.
This style of hard-coding works very naturally for any non-uniform class model of distinguishers.
Even if we work with a uniform model of distinguishers, if the length of the advice string $t \in \N$ is a constant (independent of $d$ the dimension of individuals $\X$), then for each $A \in \A$ we can define a TM that has $a \in \set{0,1}^t$ hard-coded as part of its description.
The number of distinguishers in $\A'$ grows by a factor of $2^t$.

\subsubsection{\FOUR}

The strongest notion of distinguishers we consider receive---as part of their input---the description $\ptstring$ of a circuit that computes $\pt$.
In this model, which we call \four, the distinguishers can accept or reject their sample based on nontrivial analysis of the circuit computing $\pt$, not just its evaluation on domain elements.
We assume that $\card{\ptstring} = n$ for some $n \in \N$.
\begin{definition}[\FOUR]
Fix \nature's distribution $\Dnat$.
Let $Z = \X \times \set{0,1} \times \set{0,1}^n$ for $n \in \N$.
For a class of distinguishers $\A \subseteq \set{\Z \to \set{0,1}}$ and $\eps > 0$, a predictor $\pt:\X \to [0,1]$ is $(\A,\eps)$-\four if for every $A \in \A$,
\begin{gather*}
\card{\Pr_{\Dnatsample}\lr{A(i,\os_i, \ptstring) = 1}
- \Pr_{\Dptsample}\lr{A(i,\ot_i, \ptstring) = 1}} \le \eps.
\end{gather*}
\end{definition}
There are a number of subtle technicalities in how we define \four, relating to how we encode $\ptstring$.
In particular, if we want to be able to simulate the prior notions of OI within \four, then we need to allow the complexity of the distinguishers in $\A$ to scale with the complexity of $\pt$.
Even evaluating $\pt$ on a single domain element requires that $\A$ can compute circuit evaluation.
This technicality sets \four apart from the prior notions, where it sufficed to think of the domain as fixed in dimension, and thus think of the distinguishers' complexity as fixed as well.
We discuss these issues in detail in Section~\ref{sec:beyond}, where we discuss the feasibility of OI.

\subsection{Multiple sample OI}
\label{sec:multi-sample}
Throughout this work, we focus on distinguishers that receive a single sample from nature or the modeled distribution, with varying levels of access to $\pt$.
A natural generalization of this model allows distinguishers to access multiple samples.
We define the generic variant as follows (where each of \one, \two, \three, and \four follow by allowing distinguishers the analogous degree of access to $\pt$).
\begin{definition}
Fix \nature's distribution $\Dnat$.
Let $m \in \N$ and $\Z = \left(\X \times \set{0,1}\right)$.
For a class of multi-sample distinguishers $\A_m \subseteq \set{\Z^m \to \set{0,1}}$ and $\eps > 0$, a predictor $\pt:\X \to [0,1]$ is $(\A_m,\eps)$-OI if for every $A_m \in \A_m$,
\begin{multline*}
    \Bigg\vert\Pr_{(i_1,\os_{i_1}),\hdots,(i_m,\os_{i_m}) \sim (\Dnat)^m}\lr{A_m\left((i_1,\os_{i_1}),\hdots,(i_m,\os_{i_m}); \pt \right) = 1}\\ - \Pr_{(i_1,\ot_{i_1}),\hdots,(i_m,\ot_{i_m}) \sim \Dpt^m}\lr{A_m\left((i_1,\ot_{i_1}),\hdots,(i_m,\ot_{i_m}); \pt \right) = 1}\Bigg\vert \le \eps.
\end{multline*}
\end{definition}

We leave full exploration of multi-sample-OI to future work, but make the following observation.
If the class of distinguishers we use admits a hybrid argument, then the multi-sample distinguishers' advantage can be bounded generically in terms of the single-sample advantage.
As an example, we show the following proposition for \three.
\begin{proposition}\label{prop:hybrid}
Fix \nature's distribution $\Dnat$.
Let $\A$ be the class of size-$s$ single-sample distinguishers, and for $m \in \N$ let $\A_m$ be the class of size-$s$ $m$-sample distinguishers.
Suppose we allow $\A$ pre-processing samples from $\Dnat$ and oracle-access to $\pt$.
For $\eps > 0$, if a predictor $\pt:\X \to [0,1]$ is $(\A,\eps/m)$-\three, then it is $(\A_m,\eps)$-\three.
\end{proposition}
\begin{proof}
Suppose there exists some $m$-sample distinguisher $A_m \in \A_m$ that distinguishes between nature and the model $\pt$ with advantage at least $\eps$.
We show that there is a single-sample randomized distinguisher $A \in \A$ that distinguishes between nature and the model with advantage at least $\eps/m$.
By contrapositive, if $\pt$ is $(\A,\eps/m)$-\three, then it must be $(\A_m,\eps)$-\three.

Consider the following sequence of hybrid distributions over $m$ samples, $(i_1,o_{i_1}),\hdots,(i_m,o_{i_m})$, where $\D_k = (\Dnat)^{m-k} \times \Dpt^{k}$ is a product distribution of $m-k$ independent samples from nature and $k$ samples from the model.
Note that assuming pre-processing access to samples from $\Dnat$ and oracle access to $\pt$, each $\D_k$ is sampleable.
Specifically, to obtain a sample from $\D_k$, we will draw $m$ samples from $\Dnat$, and then for each $j \in \set{m-k+1,\hdots,m}$, we resample the outcome by evaluating $\pt_{i_j}$ and then randomly drawing $\ot_{i_j} \sim \Ber(\pt_{i_j})$.

Observing that $\D_0 = (\Dnat)^m$ and $\D_m = \Dpt^m$, we can write the distinguishing probability of $A_m$ as a telescoping sum over distinguishing probabilities over the hybrid distributions.
\begin{multline*}
\Pr_{(i_1,\os_{i_1}),\hdots,(i_m,\os_{i_m}) \sim (\Dnat)^m}\lr{A_m^\pt\left((i_1,\os_{i_1}),\hdots,(i_m,\os_{i_m})\right) = 1}\\ - \Pr_{(i_1,\ot_{i_1}),\hdots,(i_m,\ot_{i_m}) \sim \Dpt^m}\lr{A_m^\pt\left((i_1,\ot_{i_1}),\hdots,(i_m,\ot_{i_m})\right) = 1}
\end{multline*}
\begin{align*}
&= \sum_{j=1}^m \left(\Pr_{(I,O) \sim \D_{j-1}}\lr{A_m^\pt(I,O) = 1} - \Pr_{(I,O) \sim \D_j}\lr{A_m^\pt(I,O) = 1}\right)\\
&\ge \eps
\end{align*}
Thus, the following randomized single-sample \three distinguisher succeeds with advantage at least $\eps/m$:
as pre-processing, sample a random index $j \sim [m]$ and draw a sample from the hybrid distribution $\D_j$;
on input $(i,o_i)$, replace the $j$th sample with the input $(i,o_i)$, and run $A_m$ on the resulting $m$-sample input.
If the input is drawn from nature, then the resulting sample is drawn from $\D_{j-1}$, whereas if the input is from the model, then the resulting sample is drawn from $\D_j$.
Thus, the distinguishing advantage of $A$ is the average distinguishing advantage between $\D_{j-1}$ and $\D_j$, or $\eps/m$.
\end{proof}

We state Proposition~\ref{prop:hybrid} for \three (and thus, by simulation, \four), due to the ease of running the hybrid argument with oracle access to $\pt$.
Note that we use circuit-size as the complexity measure for concreteness, but the argument will go through for most complexity measures of $\A_m$.
Similar hybrid arguments can also be made for \one and \two, provided the model of computation of the distinguishers admits ``hard-coding'' the outcome values $\set{\ot_{i_{m-k+1}},\hdots,\ot_{i_m}}$, and $\set{\pt_{i_1},\hdots,\pt_{i_{m-k}}}$ if needed (for \two).
In particular, for any non-uniform class of multi-sample distinguishers $\A$, there exists a class $\A'$ of single-sample distinguishers that simulates the distinguishers in $\A$ with the choices hard-coded.

\section{\PI}
\label{sec:pi}

Before discussing the properties of OI further, we turn our attention to an idealized notion of indistinguishability, which we refer to as \pin (PI).
Describing this notion requires the equivalent view on \nature's distribution $\Dnat$, where we imagine that \nature selects a marginal distribution over individuals and a true outcome \predictor $\ps$.
Distinguishers receive as input an individual-outcome pair $(i,\os_i) \sim \Dnat$ from \nature's distribution, and either \nature's prediction $\ps_i$ or the model's estimate of the parameter $\pt_i$.

We show that achieving PI may require learning \nature's \predictor $\ps$ very precisely, even when $\A$ is a very simple class of distinguishers.
This result shows that PI is generally infeasibille due to the ability to access $\ps_i$ directly: even computationally-weak PI distinguishers are incredibly powerful at distinguishing between $\ps$ and $\pt$.
In a sense, the hardness of PI motivates our focus on OI.

\paragraph{Statistical closeness through PI.}
\Pin requires that the joint distribution of such individual-outcome-prediction triples cannot be significantly distinguished by a family of algorithms $\A$.
\begin{definition}[\PI]
Fix \nature's distribution $\Dnat$. Let $\Z = \X \times \set{0,1} \times [0,1]$.
For a class of distinguishers $\A:\Z \to [0,1]$ and $\eps > 0$, a predictor $\pt:\X \to [0,1]$ satisfies $(\A,\eps)$-\PI (PI) if for every $A \in \A$,
\begin{gather*}
\card{\Pr_{\Dnatsample}\lr{A(i,\os_i,\ps_i) = 1}
- \Pr_{\Dnatsample}\lr{A(i,\os_i,\pt_i) = 1}} \le \eps.
\end{gather*}
\end{definition}
We emphasize that \pin departs from \oi in an essential way, by assuming the distinguisher may receive direct access to \nature's prediction $\ps_i$.\footnote{The assumption that $\ps$ meaningfully exists such that $\ps_i$ can be given as input to a distinguisher breaks the abstraction of $\Dnat$, but is a common assumption in the forecasting literature.
Still, this is another sense in which PI is an idealized variant of OI, because we can never actually generate individual-outcome-prediction samples from $\Dnat$.}

We show that \pin is too strong a notion of indistinguishability to be broadly useful.
Specifically, we show that using a very simple distinguisher, we can test for statistical closeness between nature's \predictor $\ps$ and the model's \predictor $\pt$.
Given the hardness of recovering individual-level predictions in statistical distance (both information-theoretic and computational), this reduction allows us to conclude that, in general, \pin is infeasible.

Consider the randomized distinguisher $A_{\ell_1}$ defined as follows.
\begin{gather*}
    A_{\ell_1}(i,o_i,p_i) = \begin{cases} 0 &\text{w.p.~} \card{o_i - p_i}\\
    1 &\text{o.w.}
    \end{cases}
\end{gather*}
We argue that if a candidate $\pt$ passes this single PI-distinguisher, it must have small statistical distance to $\ps$.
\begin{proposition}
Fix \nature's distribution $\Dnat$ and constant $\eps,\tau \ge 0$; suppose \nature's \predictor $\ps:\X \to [0,1]$ is such that $\ps = f + \delta$ for a boolean function $f:\X \to \set{0,1}$ and $\delta:\X \to [-1,1]$ where $\norm{\delta}_1 \le \tau$.
Then any $(\set{A_{\ell_1}},\eps)$-PI predictor $\pt:\X \to [0,1]$ is statistically close to $\ps$, satisfying
\begin{gather*}
    \norm{\ps - \pt}_1 \le 4\tau + \eps.
\end{gather*}
\end{proposition}
\begin{proof}
Consider the difference in probabilities of acceptance under that natural and modeled distributions.
\begin{align*}
\Pr_{\Dnatsample}&\lr{A_{\ell_1}(i,\os_i,\ps_i) = 1}
- \Pr_{\Dnatsample}\lr{A_{\ell_1}(i,\os_i,\pt_i) = 1}\\
&= \E_{\Dnatsample}\lr{1 - \card{\os_i - \ps_i}} - \E_{\Dnatsample}\lr{1 - \card{\os_i - \pt_i}}\\
&=\E_{i \sim \DX}\lr{\ps_i \cdot (\ps_i - \pt_i) + (1-\ps_i) \cdot (\pt_i - \ps_i)}\addtag\label{pi:equality}
\end{align*}
Assuming that $\pt$ is $(\set{A_{\ell_1}},\eps)$-PI, we can upper bound this quantity by $\eps$.
Under the assumption that $\ps = f + \delta$ for boolean $f$, we will lower bound the quantity in terms of $\norm{\ps - \pt}$ and $\tau$.
\begin{align*}
(\ref{pi:equality})
&= \E_{i \sim \DX}\lr{(f_i + \delta_i)\cdot(f_i + \delta_i - \pt_i) + (1-f_i - \delta_i) \cdot (\pt_i - f_i - \delta_i)}\\
&\ge \E_{i \sim \DX}\lr{f_i\cdot(f_i + \delta_i - \pt_i) + (1-f_i) \cdot (\pt_i - f_i - \delta_i)} - 2 \norm{\delta}_1\\
&\ge \E_{i \sim \DX}\lr{f_i\cdot(f_i - \pt_i) + (1-f_i) \cdot (\pt_i - f_i)} - 3 \norm{\delta}_1\\
&= \E_{i \sim \DX}\lr{\card{f_i - \pt_i}} - 3 \norm{\delta}_1\\
&\ge \norm{f - \pt}_1 - 3 \tau\\
&\ge \norm{\ps - \pt}_1 + 4 \tau
\end{align*}

Thus, in combination, we can conclude
\begin{gather*}
    \norm{\ps - \pt}_1 - 4\tau \le (\ref{pi:equality}) \le \eps
\end{gather*}
and the proposition follows.
\end{proof}
We can therefore port any hardness results for recovering $\ps$ in statistical distance to obtaining \pin.
For example, if we take $\ps$ to be a random boolean function, then $\ell_1$-recovery is information-theoretically impossible unless we observe the outcome $\os_i$ for a $1-O(\eps)$ fraction of inputs $i \in \X$.
If we restrict ourselves to relatively simple functions, $\ell_1$-recovery may be information-theoretically feasible, but computationally infeasible: for instance, if $\ps$ is a pseudorandom function, then any computationally-efficient estimate of $\pt$ will fail $(\set{A_{\ell_1}}, \eps)$-PI.

\paragraph{PI when distinguishers don't reject \naturelike histories.}

While in full generality, PI may be infeasible, we show that for a certain broad class of distinguishers ---derived from tests used in the forecasting literature---OI and PI are actually equivalent.
Recall, in \cite{sandroni2003reproducible}, the goal is to learn a predictive model that passes tests as well as nature, where the tests may be over histories of outcomes.
Importantly, in the forecasting literature, the \predictor is compared to \emph{every} possible choice of nature, and only needs to perform as well as the ``worst'' nature.

We show that for tests that accept \naturelike histories with high probability, PI and OI are essentially the same condition.
This requirement---that every possible nature must pass the tests---turns out to be very restrictive.
We formulate these tests as a restricted family of multi-sample distinguishers.
\begin{definition}[Tests that rarely reject nature]
Let $m \in \N$ and $\eps \ge 0$, and suppose $\A$ is a family of $m$-sample PI distinguishers.
We say that $\A$ rejects nature with probability at most $\eps$ if for every choice of \nature's distribution $\Dnat = \D(\ps)$ for \predictor $\ps:\X \to [0,1]$, for every $A \in \A$,
\begin{gather*}
\Pr_{(i_1,\os_{i_1}),\hdots,(i_m,\os_{i_m}) \sim \D(\ps)^m}\lr{A((i_1,\os_{i_1},\ps_{i_1}),\hdots,(i_m,\os_{i_m},\ps_{i_m})) = 1} \ge 1-\eps.
\end{gather*}
\end{definition}
Despite the fact that these PI distinguishers receive access to the true generating parameters of the model, we show the tests are very limited in their distinguishing abilities.
In fact, we will show that the equivalence between PI and OI holds for distinguishers satisfying an even weaker property, which we call the \emph{\validmodel} property.
Specifically, we will assume that for all $A \in \A$, there exists an acceptance probability $q_A$, such that for all $p:\X \to [0,1]$,
\begin{gather*}
    \card{\Pr_{(i_1,o_{i_1}),\hdots,(i_m,o_{i_m}) \sim \D(p)^m}\lr{A\left((i_1,o_{i_1},p_{i_1}),\hdots,(i_m,o_{i_m},p_{i_m})\right) = 1} - q_A~}\le \eps.
\end{gather*}
Essentially, such \validmodel distinguishers may test whether the outcomes $o_{i_j}$ are actually sampled from the model $\Ber(p_{i_j})$ over the collection of samples they receive, but cannot depend on the value of $p_{i_j}$ (independently of $o_{i_j}$).
Tests that rarely reject nature satisfy the \validmodel property with $q_A = 1$ for all $A \in \A$.

\begin{proposition}
Fix \nature's distribution $\Dnat$.
Suppose for $\eps \ge 0$, $\A$ is a family of distinguishers satisfying the \validmodel property.
Then, PI and \two over $\A$ are equivalent; specifically,
\begin{itemize}
    \item if $\pt$ is $(\A,\eps)$-\two, it is $(\A,3\eps)$-PI
    \item if $\pt$ is $(\A,\eps)$-PI, it is $(\A,3\eps)$-\two.
\end{itemize}
\end{proposition}
\begin{proof}
Suppose that $\A$ is a family satisfying the \validmodel property.
First, we show that for all $A \in \A$, by the \validmodel property, the acceptance probabilities over $\D(\pt)$ and $\Dnat = \D(\ps)$, when we feed $A$ the correct generating probabilities, must be similar.
\begin{align}
&\Pr_{(i_1,\ot_{i_1}),\hdots,(i_m,\ot_{i_m}) \sim \D(\pt)^m}\lr{A\left((i_1,\ot_{i_1},\pt_{i_1}),\hdots,(i_m,\ot_{i_m},\pt_{i_m})\right) = 1} \label{ineq:pi:1}\\
&\qquad\le q_A + \eps \label{ineq:pi:2}\\
&\qquad\le \Pr_{(i_1,\os_{i_1}),\hdots,(i_m,\os_{i_m}) \sim \D(\ps)^m}\lr{A\left((i_1,\os_{i_1},\ps_{i_1}),\hdots,(i_m,\os_{i_m},\ps_{i_m})\right) = 1} + 2\eps\label{ineq:pi:3}
\end{align}
These inequalities (and their reverse) follow directly from the \validmodel assumption on $\A$.
Using this closeness, we can show both directions of the equivalence.

First, suppose the OI distinguishing advantage is bounded by $\eps$.
\begin{multline}\label{ineq:pi:oi}
\Bigg\vert\Pr_{(i_1,\os_{i_1}),\hdots,(i_m,\os_{i_m}) \sim (\Dnat)^m}\lr{A\left((i_1,\os_{i_1},\pt_{i_1}),\hdots,(i_m,\os_{i_m},\pt_{i_m})\right) = 1}\\ -
\Pr_{(i_1,\ot_{i_1}),\hdots,(i_m,\ot_{i_m}) \sim \D(\pt)^m}\lr{A\left((i_1,\ot_{i_1},\pt_{i_1}),\hdots,(i_m,\ot_{i_m},\pt_{i_m})\right) = 1}\Bigg\vert \le \eps
\end{multline}
Applying the \validmodel inequalities (from (\ref{ineq:pi:oi})$\to$(\ref{ineq:pi:1})$\to$(\ref{ineq:pi:2})$\to$(\ref{ineq:pi:3})), we can bound the PI distinguishing advantage by $3\eps$.

Next, suppose the PI distinguishing advantage is bounded by $\eps$.
\begin{multline}\label{ineq:oi:pi}
\Bigg\vert\Pr_{(i_1,\os_{i_1}),\hdots,(i_m,\os_{i_m}) \sim (\Dnat)^m}\lr{A\left((i_1,\os_{i_1},\pt_{i_1}),\hdots,(i_m,\os_{i_m},\pt_{i_m})\right) = 1}\\ -
\Pr_{(i_1,\os_{i_1}),\hdots,(i_m,\os_{i_m}) \sim (\Dnat)^m}\lr{A\left((i_1,\os_{i_1},\ps_{i_1}),\hdots,(i_m,\os_{i_m},\ps_{i_m})\right) = 1}\Bigg\vert \le \eps
\end{multline}
Applying the \validmodel inequalities (from (\ref{ineq:pi:1})$\to$(\ref{ineq:pi:2})$\to$(\ref{ineq:pi:3})$\to$(\ref{ineq:oi:pi})), we can bound the OI distinguishing advantage by 3$\eps$.
\end{proof}

\section{OI and Evidence-Based Fairness}
\label{sec:connections}

With the formal definitions of the variants of OI in place, we turn to understanding the guarantees implied by OI.
In this section, we show a tight connection between the notion of \oi and notions of \emph{evidence-based fairness}.

\paragraph{Evidence-based fairness notions.}
We review two related notions---\emph{multi-accuracy} and \emph{multi-calibration}---which require a predictor $\pt: \X \to [0,1]$ to display ``consistency'' with $\ps$, not simply on the population $\X$ as a whole, but also on structured subpopulations of $S \subseteq \X$.
The notions are parameterized by a class of subpopulations $\C \subseteq \set{0,1}^\X$, which controls the strength of the guarantee:  the richer the class of subpopulations, the stronger the consistency requirement.

For conceptual and algorithmic reasons, it is natural to define the collection $\C$ in terms of a concept class (e.g.,\ conjunctions, halfspaces, decision trees), where each function in the class $c:\X \to \set{0,1}$ defines the characteristic function for a subpopulation $S_c \in \C$.
We review the definitions of \ma and \mc, and refer the reader to \cite{hkrr,kim} for more in-depth coverage of the definitions and algorithms for achieving the notions.

We start with the notion of \ma.
Given a collection of subpopulations $\C$, \ma requires that a predictor $\pt$ reflect the expectations of $\ps$ correctly over each subpopulation $S \in \C$.
\begin{definition}[\MA \cite{hkrr}]
\label{def:ma}
Fix a distribution over individuals $\DX$ and \nature's \predictor $\ps:\X \to [0,1]$.
For a collection of sets $\C \subseteq \set{0,1}^\X$ and $\alpha \ge 0$, a predictor $\pt: \X \to [0,1]$ satisfies $(\C, \alpha)$-\ma if for every $S \in \C$, 
\begin{equation}
\label{eqn:ma}
    \card{~\E_{\substack{i \sim \DX}}\lr{ \ps_i \given i\in S}
- \E_{\substack{i \sim \DX}}\lr{\pt_i \given i\in S }~} \le \alpha
\end{equation}
\end{definition}
Predictors that satisfy \ma cannot introduce significant bias in their predictions over any subpopulation defined within $\C$.
Still, because we average over $S$, multi-accuracy may allow the predictor to make distinctions between individuals with similar $\ps$ values.
For instance, if every individual $i \in \X$ has $\ps_i = 1/2$, then a random boolean function that predicts $\pt_i \in \set{0,1}$ for each $i \in \X$ will satisfy \ma for $\C$ of bounded complexity.

\Mc prevents this type of disparate treatment of similar individuals by requiring the predictor $\pt$ to be calibrated with respect to $\ps$ over each $S \in \C$.
Here, a set of predictions is calibrated if amongst the individuals $i \in \X$ who receive prediction $\pt_i = v$, their actual expectation is $v$.
Intuitively, calibration goes further than accuracy in expectation and requires that the predictions can be meaningfully interpreted as conditional probabilities.

Technically, to reason about approximate calibration in a way that provides strong guarantees and is practically realizable, we need to work with discretized predictors.
Recall that for a predictor $\pt$, we use $\supp_S(\pt) = \set{v \in [0,1] : \Pr_{i \sim \DX}\lr{\pt_i = v \given i \in S} > 0}$ to denote the support of $\pt$ on $S \subseteq \X$.

\begin{definition}[\MC \cite{hkrr}]
\label{def:mc}
Fix a distribution over individuals $\DX$ and \nature's \predictor $\ps:\X \to [0,1]$.
For a collection of sets $\C \subseteq \set{0,1}^\X$  and parameters $\alpha > 0$, a predictor $\pt: \X \to [0,1]$ satisfies $(\C,\alpha)$-\mc if for every set $S \in \C$ and supported value $v \in \supp_S(\pt)$ such that $\Pr_{i \sim \DX}\lr{\pt_i = v \given i\in S} \geq \alpha / \card{\supp_S(\pt)}$,
\begin{equation}
\label{eqn:mc}
    \card{~\E_{i \sim \D}\lr{\ps_i \given \pt_i = v \land i\in S} - v~} \leq \alpha.
\end{equation}
\end{definition}

Intuitively, by conditioning on the value of $\pt$ in addition to membership in $S$, multi-calibration prevents a predictor from treating similar individuals differently.
If a multi-calibrated predictor says that the value $\pt_i = v$, then the ``probability'' of their outcome being $1$ can be thought of as $v$, where the probability is taken over the randomness in the outcome and the random selection of a member of any $S \in \C$ for which $i \in S$.

We include two technical remarks about Definition~\ref{def:mc}:
\begin{itemize}
\item Observe that the definition of \mc excludes from consideration level sets where $\Pr_{i \sim \DX}\lr{\pt_i = v \given i \in S} < \alpha/\card{\supp_S(\pt)}$.
This exclusion is a technicality needed in order to learn multi-calibrated predictors from a small set of training data.
Because the tolerance is scaled by the support size of $\pt$, the total fraction of any subpopulation $S \in \C$ that can be excluded is $\alpha$.
Effectively, this means that predictors have to be confident in their predictions---a more refined predictor with many supported values must satisfy the multi-calibration constraints on smaller slices.
\item The notion of approximate calibration used in Definition~\ref{def:mc} enforces that the expectation of $\pt$ and $\ps$, conditioned on $ i \in S$ and $\pt_i = v$, are the same up to a tolerance of $\alpha$.
This relative error guarantee (which becomes more stringent for smaller sets) is deliberate and best captures the intuition  that calibrated predictions should ``mean what they say.''\footnote{A natural alternative notion of approximation would allow for \emph{absolute} error, where the slack on the conditional expectation in (\ref{eqn:mc}) is scaled as $\alpha/\Pr_{i \sim \DX}\lr{\pt_i = v \land i \in S}$.
While easier to work with algorithmically, this constraint provides minimal guarantees of consistency with $\ps$.
For instance, the predictor that randomly assigns individuals one of $1/\alpha$ values will, with high probability, have $\Pr_{i \sim \DX}\lr{\pt_i = v \land i \in S} \approx \alpha$ for all $S \in \C$, and thus, each of the constraints would be vacuously satisfied.}
\end{itemize}

\paragraph{Relating OI to \ma and \mc.}
In the next sections, we demonstrate a tight connection between \one and \ma, and \two and \mc.
We will show that we can implement the OI notion as an instance of multi-group fairness, and vice versa.
In particular, given any distinguisher class $\A$, we can construct a corresponding class of subpopulations $\C_\A$, such that multi-calibration over $\C_\A$ implies \two over $\A$; and given any class of subpopulations $\C$, we can construct a corresponding distinguisher class $\A_\C$ such that \two over $\A_\C$ implies \mc over $\C$ (with small loss in the approximation parameters).
In fact, the transformation between distinguishing algorithms $\A$ and subpopulations $\C$ is tight enough that if we repeat the construction a constant number of times, the classes hit a ``fixed point'' (e.g.,\ $\C_{\A_\C} = \C$).

An important quantity in the analysis of \oi is the advantage of a distinguisher $A$ with respect to a particular predictor $p$. For example, in \two:

\begin{equation*}
    \Delta_A(p) \triangleq  \card{\Pr_{\Dnatsample}\lr{A(i,\os_i,p_i) = 1}
- \Pr_{(i,o_i) \sim \Dmod{p}}\lr{A(i,o_i,p_i) = 1}}
\end{equation*}

It will be natural to define analogous quantities for the evidence-based fairness notions we mentioned above; we refer to these as the \emph{\ma violation} and \emph{\mc violation}, respectively. Given a set $S$ and possibly a value $v \in [0,1]$, these are defined as follows:

\begin{align*}
\nabla_{S}^{\textbf{MA}}(p) &\triangleq
\card{\E_{i \sim \DX}\lr{\ps_i \given i \in S} - \E_{i \sim \DX}\lr{p_i \given i \in S}~}\\
&= \card{\Pr_{\Dnatsample}\lr{\os_i = 1 \given i\in S}
- \Pr_{(i,o_i) \sim \Dmod{p}}\lr{o_i = 1 \given  i\in S}~}
\end{align*}

\begin{align*}
\nabla_{S,v}^{\textbf{MC}}(p) &\triangleq 
\card{\E_{i \sim \DX}\lr{\ps_i \given p_i = v \land i \in S} - v~}\\
&= \card{\Pr_{\Dnatsample}\lr{\os_i = 1 \given p_i = v \land i\in S} - v~} \end{align*}

Since the subscript allows us to distinguish between the two violations, we will typically just write $\nabla(p)$, and the corresponding fairness notion will be clear from context: $\nabla_{S}$ for the \ma violation, and $\nabla_{v,S}$ for the \mc violation.

\subsection{\ONE and \MA}

In this section we show that \one and \ma  are very closely related. In particular, we prove that each can be implemented using the other, by an appropriate choice of the distinguisher class (resp., of the collection of sets) and accuracy parameters.

\begin{theorem}
Fix \nature's distribution $\Dnat$.  Let $\Z = \X \times \set{0,1}$.
\begin{itemize}
    \item \emph{Implementing \ma using OI}:
    Suppose $\C \subseteq \set{0,1}^\X$ is a collection of subpopulations, and let $\gamma_C \ge \min_{S \in \C} \Pr_{i \sim \DX}\lr{i \in S}$.
    Then there exists a family of distinguishers $\A_{\C} \subseteq \set{\Z \to \set{0,1}}$ such that for any $\alpha > 0$,
if a predictor $\pt:\X \to [0,1]$ satisfies $(\A_{\C}, \alpha\cdot\gamma_\C)$-\one, then $\pt$ satisfies $(\C,\alpha)$-\ma.
    \item \emph{Implementing OI using \ma}: Suppose $\A \subseteq \set{\Z \to \set{0,1}}$ is a family of deterministic distinguishers.
    Then there exists a collection of subpopulations $\C_{\A} \subseteq \set{0,1}^\X$ such that for any $\eps > 0$, if a predictor $\pt$ satisfies $(\C_{\A}, \eps/2)$-\ma, then $\pt$ satisfies $(\A,\eps)$-\one.
\end{itemize}
\end{theorem}

Note that proof reveals the additional property that the complexity of $\C$ and $\A$ are essentially equivalent.
That is, for each $S \in \C$, we construct a distinguisher $A_S \in \A_\C$ such that the complexity of evaluating set membership $\1\lr{i \in S}$ is nearly identical to that of computing $A_S(i,b)$ for $b \in \set{0,1}$.
\begin{proof}
We prove each direction separately.

\emph{Implementing \ma using OI}. 
Fix a collection of sets $\C$. We define a collection of $\card{\C}$ distinguishers $\A_\C = \set{A_S: \, S\in \C}$, where for each $S \in \C$, the distinguisher $A_{S}$ is defined as:

\begin{equation*}
    A_{S}(i,b)=\textbf{1}\lr{i\in S\land b=1}
\end{equation*}

For any predictor $\pt$ and set $S \in \C$, we can relate the \ma violation $\nabla_S(\pt)$ to the distinguishing advantage $\Delta_{A_{S}}(\pt)$, as follows.

\begin{align*}
\nabla_{S}(\pt)&=\card{\Pr_{\Dnatsample}\lr{\os_i = 1 \given i\in S}
- \Pr_{\Dptsample}\lr{\ot_i = 1 \given  i\in S}} \\
&= \frac{1}{\Pr_{i\sim \DX}[i \in S]}\cdot \card{\Pr_{\Dnatsample}\lr{\os_i = 1 \land i\in S}
- \Pr_{\Dptsample}\lr{\ot_i = 1 \land  i\in S}} \\ 
&= \frac{1}{\Pr_{i\sim \DX}[i \in S]}\cdot \card{\Pr_{\Dnatsample}\lr{A_{S}(i,\os_i)=1}
- \Pr_{\Dptsample}\lr{A_{S}(i,\ot_i)=1}} \\ 
&= \frac{1}{\Pr_{i\sim \DX}[i\in S]}\cdot\Delta_{A_{S}}(\pt)
\end{align*}

Therefore, if $\pt$ satisfies $(\A_{\C}, \alpha\cdot\gamma_\C)$-\one, we derive the following string of inequalities.
\begin{equation*}
    \nabla_{S}(\pt)=\frac{\Delta_{A_{S}}(\pt)}{\Pr_{i\sim \DX}[i\in S]}\leq\frac{\Delta_{A_{S}}(\pt)}{\gamma_\C}\leq\frac{\alpha\cdot \gamma_\C}{\gamma_\C}=\alpha
\end{equation*}

Thus, for every $S \in \C$ we have $\nabla_{S}(\pt) \leq \alpha$, and we conclude that $\pt$
satisfies $(\C,\alpha)$-\ma.

\emph{Implementing OI using \ma}. Fix a collection of deterministic distinguishers $\A$. We define a collection of $2\cdot \card{\A}$ sets $\C_\A = \set{S_{(A,b)}: \, A \in \A,\, b\in \set{0,1}}$, where for each $A \in \A$ and $b \in \set{0,1}$, the subpopulation $S_{(A,b)}$ is defined as
\begin{equation*}
    S_{(A,b)}=\set{i\in \X: \,\, A(i,b)=1}
\end{equation*}

For any predictor $\pt$ and distinguisher $A \in \A$, we can relate the distinguishing advantage $\Delta_{A}(p)$ to the multi-accuracy violations $\nabla_{S_{(A,1)}}$ and $\nabla_{S_{(A,0)}}$, as follows.
\begin{align*}
&\Delta_A(\pt)
= \card{\Pr_{\Dnatsample}\lr{A(i,\os_i) = 1}
- \Pr_{\Dptsample}\lr{A(i,\ot_i) = 1}} \\
&= \card{\sum_{b \in \set{0,1}}\left(\Pr_{\Dnatsample}\lr{A(i,b) = 1\land \os_i = b}
- \Pr_{\Dptsample}\lr{A(i,b) = 1 \land \ot_i = b}\right)} \\
&\le \sum_{b \in \set{0,1}} \card{\Pr_{\Dnatsample}\lr{A(i,b) = 1\land \os_i = b}
- \Pr_{\Dptsample}\lr{A(i,b) = 1 \land \ot_i = b}} \\
&= \sum_{b \in \set{0,1}} \card{\Pr_{\Dnatsample}\lr{i \in S_{(A,b)}\land \os_i = b}
- \Pr_{\Dptsample}\lr{i \in S_{(A,b)} \land \ot_i = b}} \\
&= \sum_{b \in \set{0,1}} \Pr_{i\sim \DX}\lr{i \in S_{(A,b)}}\cdot \card{\Pr_{\Dnatsample}\lr{\os_i = b \given i \in S_{(A,b)}}
- \Pr_{\Dptsample}\lr{\ot_i = b \given i \in S_{(A,b)}}} \\
&= \sum_{b \in \set{0,1}} \Pr_{i\sim \DX}\lr{i \in S_{(A,b)}}\cdot \nabla_{S_{(A,b)}}(\pt) \\
&\le \nabla_{S_{(A,0)}}(\pt)+\nabla_{S_{(A,1)}}(\pt)
\end{align*}
Therefore, if $\pt$ satisfies $(\C_\A,\eps/2)$-\ma, then we have $\Delta_A(\pt) \leq \eps$ for every $A \in \A$; we conclude that $\pt$ satisfies $(\A,\eps)$-\one.
\end{proof}

\subsection{\TWO and \MC}

In this section we establish a similar relationship, this time proving that \two and \mc can each be implemented using the other.  One technical subtlety is that unlike the previous level, we need to work with some fixed discretization of the predictor in question. Formally, for any $m \in \N$, we use $\pg$ to denote the ``rounded'' version of a predictor $p: \X\to [0,1]$, with respect to the grid $G_m = \set{\frac{1}{2m}, \frac{2}{2m}, \dots, \frac{2m-1}{2m}}$ that partitions the interval $[0,1]$ into $m$ ``bins'' of width $1/m$:
\begin{equation*}
    \forall i \in \X: \qquad \pg_i \equiv u(p_i)
\end{equation*}
where for $v \in [0,1]$, $u(v) \in G_m$ is the closest value to $v$ that is on the grid: $u(v) \equiv \argmin_{u \in G_m}\card{u-v}$. Note that by definition, for every $i \in \X$, $\card{\pg_i - p_i} \equiv \card{u(p_i) - p_i}\le 1/2m$.

\begin{theorem}
Fix nature's distribution $\Dnat$. Let $\Z = \X \times \set{0,1} \times [0,1]$. 
\begin{enumerate}
    \item \emph{Implementing \mc using OI:} Suppose $\C \subseteq \set{0,1}^\X$ is a collection of subpopulations, and let $\gamma_C \ge \min_{S \in \C} \Pr_{i \sim \DX}\lr{i \in S}$. Then for every   $m \in \N$, there exists a family of distinguishers $\A_{\C,m} \subseteq \set{\Z \to \set{0,1}}$ such that for any $\alpha > 0$ and for every predictor $p: \X \to [0,1]$, if 
     $\pg$  satisfies $(\A_{\C,m}, \alpha^2 \cdot \gamma_\C / m)$-\two, then $\pg$ satisfies $(\C, \alpha)$-\mc.
    
    \item \emph{Implementing OI using \mc:} Suppose $\A \subseteq \set{\Z \to \set{0,1}}$ is a family of deterministic distinguishers. Then for every  $m \in \N$, there exists a collection of subpopulations $\C_{\A,m} \subseteq \set{0,1}^\X$ such that for any $\eps > 0$ and  for every predictor $\pt: \X \to [0,1]$, if  $\pg$ satisfies $( \C_{\A,m},\eps/4)$-\mc, then $\pg$ satisfies $(\A,\eps)$-\two.
\end{enumerate}
\end{theorem}

\begin{proof}
We prove each direction separately.

\emph{Implementing \mc using OI.} Fix $\C$ and $m \in \N$. We define a family of $\card{\C}\cdot m$ distinguishers, $\A_{\C,m} = \set{A_{S,u} \, \, | \,\, u \in G_m, S \in \C}$, where:
\begin{equation*}
    A_{u,S}(i,b,v) = \textbf{1}\lr{i \in S \land b=1 \land \card{u-v} \leq 1/2m}
\end{equation*}

For any subpopulation $S \in \C$ and value $ u \in G_m$, we can relate the distinguishing advantage $\Delta_{A_{u,S}}(\pg)$ to the \mc violation $\nabla_{u,S}(\pg)$, as follows: 

\begin{align}
    & \Delta_{A_{u,S}}(\pg) \nonumber \\ &= \card{\Pr_{\Dnatsample}\lr{A_{u,S}(i,\os_i,\pg_i) = 1}
- \Pr_{\Dpgsample}\lr{A_{u,S}(i,\ot_i,\pg_i) = 1}} \nonumber \\
&= \card{\Pr_{\Dnatsample}\lr{i \in S \land \os_i=1 \land \card{u-\pg_i} \leq 1/2m}
- \Pr_{\Dpgsample}\lr{i \in S \land \ot_i=1 \land \card{u-\pg_i} \leq 1/2m}} \nonumber \\
&= \card{\Pr_{\Dnatsample}\lr{i \in S \land \os_i=1 \land \pg_i = u}
- \Pr_{\Dpgsample}\lr{i \in S \land \ot_i=1 \land \pg_i = u}} \nonumber \\
&= \Pr_{i\sim\DX}[\pg_i = u \land i\in S]\cdot  \card{\Pr_{\Dnatsample}\lr{\os_i = 1  \given  \pg_i = u \land i\in S}
- \Pr_{\Dpgsample}\lr{\ot_i = 1 \given \pg_i = u \land i\in S}} \nonumber \\ 
 &= \Pr_{i\sim \DX}[\pg_i = u \land i\in S]\cdot \nabla_{u,S}(\pg) \nonumber
\end{align}

Suppose $\pg$ satisfies $(\A_{\C,m},\eps)$-\two for $\eps = \frac{\alpha^2 \cdot \gamma_\C}{m}$.
To prove $\pg$ satisfies \mc, we must  bound the \mc violation $\nabla_{u,S}(\pg)$ whenever $u,S$ are such that $\Pr_{i\sim \DX}\lr{\pg_i = u \given i\in S} \geq \alpha/\supp_S(\pg)$. 

Suppose $u,S$ satisfy $\Pr_{i\sim \DX}\lr{\pg_i = u \given i\in S} \geq \alpha/\supp_S(\pg) \geq \alpha/m$. Note that this implies

\begin{equation*}
    \Pr_{i\sim \DX}\lr{\pg_i = u \land i\in S} = \Pr_{i\sim \DX}\lr{\pg_i = u \given i\in S} \cdot \Pr_{i\sim \DX}\lr{i\in S} \geq \frac{\alpha \cdot \gamma_{\C}}{m}
\end{equation*}

Using the above, we obtain:

\begin{equation*}
    \nabla_{u,S}(\pg) =  \frac{\Delta_{A_{u,S}}(\pt)}{\Pr_{i\sim \DX}[\pg_i = u \land i\in S]} \leq \frac{\alpha^2 \cdot \gamma_\C / m}{\alpha \cdot \gamma_\C / m} = \alpha
\end{equation*}
Which guarantees that $\pg$ satisfies $(\C,\alpha)$-\mc, as required.

\emph{Implementing OI using \mc.} 
Fix a collection of deterministic distinguishers $\A$ and $m \in \N$. We define a family of $2m\cdot \card{\A}$ sets, $\C_{\A,m} = \set{S_{A,b,u} \,\, \given \,\, A \in \A, b \in \set{0,1}, u \in G_m}$, where:

\begin{equation*}
    S_{A,b,u} = \set{i \in \X: \,\, A(i,b,u)=1}
\end{equation*}

To simplify notation, in what follows we denote

\begin{equation*}
    \gamma_{A,b,u} \triangleq  \Pr_{i \sim \DX}[\pg_i = u \land i \in S_{A,b,u}]
\end{equation*}

For any distinguisher $A \in \A$, we can relate the distinguishing advantage $\Delta_{A}(\pg)$ to the \mc violations $\nabla_{S_{A,b,u}}(\pg)$ for $ b\in \set{0,1}$ and $u \in G_m$, as follows:

\begin{align*}
     & \Delta_{A}(\pg) =  \\ &= \card{\Pr_{\Dnatsample}\lr{A(i,\os_i,\pg_i) = 1}
- \Pr_{\Dpgsample}\lr{A(i,\ot_i,\pg_i) = 1}} \\
&= \card{\sum_{u, b} \lr{ \Pr_{\Dnatsample}\lr{\os=b \land \pg_i = u \land A(i,b, u)= 1}
- \Pr_{\Dpgsample}\lr{\ot_i=b \land \pg_i = u \land A(i,b, u)= 1}}} \\
&= \card{\sum_{u, b} \gamma_{A,b,u} \cdot \lr{ \Pr_{\Dnatsample}\lr{\os=b \given \pg_i = u \land i \in S_{A,b,u}}
- \Pr_{\Dpgsample}\lr{\ot_i=b \given \pg_i = u \land i \in S_{A,b,u}}}} \\
&\leq  \sum_{u, b}\gamma_{A,b,u} \cdot \card{\lr{ \Pr_{\Dnatsample}\lr{\os=b \given \pg_i = u \land i \in S_{A,b,u}} - \Pr_{\Dpgsample}\lr{\ot_i=b \given \pg_i = u \land i \in S_{A,b,u}}}} \\
&= \sum_{u, b}\gamma_{A,b,u} \cdot \card{\lr{ \Pr_{\Dnatsample}\lr{\os=1 \given \pg_i = u \land i \in S_{A,b,u}} - \Pr_{\Dpgsample}\lr{\ot_i=1 \given \pg_i = u \land i \in S_{A,b,u}}}} \\
&= \sum_{u, b}\gamma_{A,b,u} \cdot \nabla_{S_{A,b,u}}(\pg)
\end{align*}

Suppose $\pg$ satisfies $(\C,\alpha)$-\mc, for some $\alpha > 0$. Recall that, by definition, \mc provides a bound on the \mc violation $\nabla_{S_{A,b,u}}(\pg)$ whenever $A,b,u$ are such that $\Pr_{i \sim \DX}\lr{\pg_i = u \given i \in S_{A,b,u}}< \alpha / \supp_{S_{A,b,u}}(\pg)$. We therefore proceed by splitting 
the above sum into two parts, which we denote by $I_{\text{small}}$ and $I_{\text{large}}$. In particular,  for $u \in G_m$ and $b \in \set{0,1}$:

\begin{equation*}
    (u,b) \in I_{\text{small}} \iff  \Pr_{i \sim \DX}\lr{\pg_i = u \given i \in S_{A,b,u}}< \alpha /m
\end{equation*}

And we let $I_{\text{large}} = G_m \times \set{0,1} - I_{\text{small}}$. As mentioned, for these ``large'' sets, \mc guarantees an upper bound of $\alpha$ on the \mc violation:

\begin{align*}
    \sum_{(u,b) \in I_{\text{large}}}\gamma_{A,b,u}\cdot \nabla_{S_{A,b,u}}(\pg) &= \sum_{u: (u,0) \in I_{\text{large}}}\gamma_{A,0,u}\cdot \nabla_{S_{A,0,u}}(\pg) + \sum_{u: (u,1) \in I_{\text{large}}}\gamma_{A,1,u}\cdot \nabla_{S_{A,1,u}}(\pg) \\
    &\le \max_{u: (u,0)\in I_{\text{large}}}\nabla_{S_{A,0,u}}(\pg)+ \max_{u: (u,1)\in I_{\text{large}}}\nabla_{S_{A,1,u}}(\pg) \\
    &\le 2\alpha
\end{align*}

Where the above uses the fact that  $\sum_{i}{a_ix_i \leq \max_i{x_i}}$ whenever $\sum_i{a_i}\leq 1$.

 For the ``small sets'' we use a trivial upper bound on the \mc violation, but exploit the fact that the overall mass there is small:

\begin{align*}
    \sum_{(u,b) \in I_{\text{small}}}\gamma_{A,b,u}\cdot \nabla_{S_{A,b,u}}(\pg) &\le \sum_{(u,b) \in I_{\text{small}}}\gamma_{A,b,u}
    \\ 
    &= \sum_{(u,b) \in I_{\text{small}}}\Pr_{i \sim \DX}[\pg_i = u \land i \in S_{A,b,u}] \\
    &\le \sum_{(u,b) \in I_{\text{small}}}\Pr_{i \sim \DX}\lr{\pg_i = u \given i \in S_{A,b,u}} \\
    &\le \sum_{(u,b) \in I_{\text{small}}}{\alpha/m} \\
    &\le 2\alpha
\end{align*}

Put together, we obtain

\begin{align*}
    \Delta_{A}(\pg) &\leq \sum_{(u,b) \in I_{\text{small}}}\gamma_{A,b,u}\cdot \nabla_{S_{A,b,u}}(\pg) + \sum_{(u,b) \in I_{\text{large}}}\gamma_{A,b,u}\cdot \nabla_{S_{A,b,u}}(\pg)  \le 4\alpha
\end{align*}

Taking $\alpha = \eps/4$, we conclude that $\pg$ satisfying $(\C, \eps/4)$ \mc guarantees that for every $A \in \A$, $\Delta_{A}(\pg) \le \eps$. This means 
 $\pg$ satisfies $(\A, \eps)$-\two.

\end{proof}

\section{\OI beyond \MC}
\label{sec:beyond}

In this section, we give a generic construction that shows how to construct predictors satisfying \oi (for any level of access to $\pt$).
We analyze the resulting complexity of $\pt$ for each notion of OI.
Informally, for \one and \two, the construction establishes that the existence of OI predictors whose complexity scales linearly with the complexity of $\A$ and inverse polynomially on $\eps$.
Note that by the reductions of Section~\ref{sec:connections} and prior algorithmic results of \cite{hkrr}, we could obtain predictors for \one and \two with similar guarantees through reduction to obtaining \ma or \mc, respectively.

For \three and \four, the construction establishes novel bounds on the complexity required to obtain each notion.
We prove that, in each case, OI predictors $\pt$ exist in complexity that scales independently of the complexity of \nature's \predictor $\ps$, but with considerably worse dependence on the approximation parameter, growing exponentially with $1/\eps$ for \three and doubly-exponentially with $1/\eps$ for the most general form of \four.
In Section~\ref{sec:lowerbound}, we show evidence that some exponential dependence on $1/\eps$ is necessary under plausible complexity conjectures.

Our proof is constructive: we describe an algorithm (Algorithm~\ref{alg:construction}) that iteratively refines a candidate predictor, until no more distinguishers $A \in \A$ can distinguish between the distributions induced by the model and $\ps$.
The resulting predictor, naturally, satisfies \oi.
Thus, the remaining analysis bounds the number of iterations necessary to terminate, and how the complexity of $\pt$ scales with each additional iteration.
We describe the algorithm in a generic manner that works for distinguishers with any degree of access to the candidate model.

The iterative algorithm can be viewed as a variant of gradient descent or boosting, following the works of \cite{ttv,hkrr,kgz}.
In this exposition, we focus on demonstrating the feasibility of the resulting OI predictors, not on the feasibility of implementing the algorithm efficiently (statistically or computationally) which is the focus of \cite{hkrr}.
We discuss issues of sample complexity and running time briefly in Section~\ref{sec:beyond:samples}.

\begin{algorithm}[t!]
\caption{\label{alg:construction} \textsc{Construct \OI}}
\hrulefill

$p^{(0)}(\cdot) \gets 1/2$\hfill\texttt{//initialize to constant predictor}

\textbf{repeat for} $t = 0,1,\hdots$
\begin{algorithmic}
\STATE \textbf{for all} $A \in \A$.

~~~$\Delta_{A}^{(t)} \gets \Pr\limits_{\Dnatsample}\lr{A(i,\os_i; p^{(t)}) = 1}
- \Pr\limits_{(i,\ot_i) \sim \Dmod{p^{(t)}}}\lr{A(i,\ot_i; p^{(t)}) = 1}$

\textbf{end for}\hfill\texttt{// distinguishing advantage of A}
\IF{$\exists A \in \A$ s.t.\ $\card{\Delta_A^{(t)}} > \eps$}
\STATE $\delta_A^{(t)}(\cdot) \gets A(\cdot,1;p^{(t)}) - A(\cdot,0;p^{(t)})$\hfill\texttt{// update defined by delta}
\STATE $p^{(t+1)}(\cdot) \gets \pi_{[0,1]}\left(p^{(t)}(\cdot) + \frac{\Delta_A^{(t)}}{2} \cdot \delta_A^{(t)}(\cdot)\right)$\hfill\texttt{// project onto interval [0,1]}
\STATE \textbf{continue}\hfill\texttt{// next iteration upon update}
\ENDIF
\RETURN $\pt = p^{(t)}$\hfill\texttt{// return when no A distinguishes}
\end{algorithmic}
\hrulefill
\end{algorithm}

Consider Algorithm~\ref{alg:construction}.
The procedure starts with an (arbitrary) initial guess of the constant function $p^{(0)}_i = 1/2$ for all $i \in \X$.
Then, in each iteration $t$, we check if there is any $A \in \A$ that distinguishes \nature's distribution $\Dnat$ and  the \predictor's distribution $\Dmod{p^{(t)}}$.
If there is no such $A \in \A$, then the procedure terminates---by construction this $\pt = p^{(t)}$ is $(\A,\eps)$-OI.
If there is some $A \in \A$ for which $\card{\Delta_A^{(t)}} > \eps$, then we perform an additive update to $p^{(t)}$ that is designed to bring $p^{(t+1)}$ closer to the optimal \predictor $\ps$.

\paragraph{Bounding the iteration complexity.}
First, we show that the procedure is guaranteed to terminate after a finite number of updates.
\begin{lemma}
\label{lem:iterations}
Algorithm~\ref{alg:construction} terminates after at most $T \le O(1/\eps^2)$ iterations.
\end{lemma}
\begin{proof}
Fixing \nature's distribution $\Dnat$, recall we use $\ps:\X \to [0,1]$ to denote \nature's \predictor (i.e., the true conditional probability distribution of outcomes).
The iteration complexity follows by a potential argument.
For a predictor $p:\X \to [0,1]$, we define a potential function $\phi(p)$ to track the mean squared error to $\ps$.
\begin{gather*}
    \phi(p) = \E_{i \sim \DX}\lr{(\ps_i - p_i)^2}
\end{gather*}
The potential is bounded $\phi(p) \le 1$ for all predictors $p$ (including the initial constant predictor).
Thus, to bound the number of iterations by $T \le O(1/\eps^2)$, we show that after every update the potential drops significantly; specifically, $\phi(p^{(t)}) - \phi(p^{(t+1)}) \ge \Omega(\eps^2)$.
\begin{align*}
\phi\left(p^{(t)}\right) - \phi\left(p^{(t+1)}\right)
&= \E_{i \sim \DX}\lr{\left(\ps_i - p^{(t)}\right)^2 - \left(\ps_i - p^{(t+1)}\right)^2}\\
&\ge\E_{i \sim \DX}\lr{\left(\ps_i - p^{(t)}\right)^2 - \left(\ps_i - p^{(t)} - \frac{\Delta_A^{(t)}}{2}\cdot \delta_A^{(t)}(i)\right)^2}\\
&=\Delta_A^{(t)}\cdot\E_{i \sim \DX}\lr{\left(\ps_i - p^{(t)}_i\right) \cdot \delta_A^{(t)}(i)} - \frac{(\Delta_A^{(t)})^2}{4}\cdot\E_{i \sim \DX}\lr{\delta_A^{(t)}(i)^2}\\
&\ge \Delta_A^{(t)} \cdot\E_{i \sim \DX}\lr{\left(\ps_i - p^{(t)}_i\right) \cdot \delta_A^{(t)}(i)} - \frac{(\Delta_A^{(t)})^2}{4}
\end{align*}
We complete the proof by showing that if the update in Algorithm~\ref{alg:construction} occurs based on $A \in \A$, then $\E_{i \sim \DX}\lr{(\ps_i - p^{(t)}_i) \cdot \delta_A^{(t)}(i)} = \Delta_A^{(t)} \ge \eps$, so the overall progress is at least $3\eps^2/4$.

\begin{align*}
\Delta_A^{(t)} &=
\Pr_{\Dnatsample}\lr{A(i,\os_i;p^{(t)}) = 1} - \Pr_{(i,\ot_i) \sim \Dmod{p^{(t)}}}\lr{A(i,\ot_i;p^{(t)}) = 1}\\
&= \sum_{b \in \set{0,1}}\left(\Pr_{\Dnatsample}\lr{A(i,b;p^{(t)}) = 1 \land \os_i = b}
-\Pr_{(i,\ot_i) \sim \Dmod{p^{(t)}}}\lr{A(i,b;p^{(t)}) = 1 \land \ot_i = b}\right)\\
&=\sum_{b \in \set{0,1}} \left(\E_{\Dnatsample}\lr{A(i,b;p^{(t)}) \cdot \1\lr{\os_i = b}} - \E_{(i,\ot_i) \sim \Dmod{p^{(t)}}}\lr{A(i,b;p^{(t)}) \cdot \1\lr{\ot_i = b}} \right)\\
&= \sum_{b \in \set{0,1}}\E_{i\sim \DX}\lr{A(i,b;p^{(t)}) \cdot \left(\Pr_{\os_i \sim \Ber(\ps_i)}\lr{\os_i = b} -  \Pr_{\ot_i \sim \Ber(p^{(t)}_i)}\lr{\ot_i = b}\right)}\\
&= \E_{i \sim \DX}\lr{A(i,1;p^{(t)}) \cdot \left(\ps_i - p^{(t)}_i\right) + A(i,0;p^{(t)})\cdot\left((1-\ps_i) - (1-p^{(t)}_i)\right)}\\
&= \E_{i \sim \DX}\lr{\left(\ps_i - p^{(t)}_i\right) \cdot \left(A(i,1;p^{(t)}) - A(i,0;p^{(t)})\right)}\\
&=\E_{i \sim \DX}\lr{\left(\ps_i - p^{(t)}_i\right) \cdot \delta_A^{(t)}(i)}
\end{align*}

Thus, each iteration causes the potential to drop by at least $\Omega(\eps^2)$, and the number of iterations is bounded by $T \le O(1/\eps^2)$.
\end{proof}

\paragraph{Bounding the \predictor's complexity.}
In the remainder of this section, we discuss how to actually implement the updates to build up the circuit that computes $\pt$.
At a high level, we maintain a circuit to compute $p^{(t)}$ in each iteration.
Initially, $p^{(0)}$ computes a constant function.
Then, at each iteration, we build $p^{(t+1)}$ from $p^{(t)}$ by constructing a new circuit that computes the addition of $p^{(t)}$ and (a scaling of) $\delta^{(t)}_A$.
Importantly, $\delta^{(t)}_A$ can be computed using two evaluations of some $A \in \A$; hence, the size of the circuit computing the eventual $\pt$ will depend on the circuit complexity of the family of distinguishers $\A$.

For \one and \two, the bounds on circuit complexity follow directly from this overview.  We overview this analysis next.
When considering \three and \four, the distinguishers at each iteration depend non-trivially on the current predictor in question.
As such, in each iteration, we need to incorporate $p^{(t)}$, not only to maintain the existing predictions, but also to compute $A^{p^{(t)}}$ itself within the computation of $\delta^{(t)}_A$.
For each variant, we describe how to implement these distinguishers in a way that we can bound the complexity of the eventual $\pt$ produced.

\paragraph{Concrete complexity assumptions.}
To discuss the complexity needed to represent predictors satisfying OI, we must fix a model of computation for the distinguishers.
We analyze Algorithm~\ref{alg:construction} assuming that each $A \in \A$ is implemented by boolean circuits satisfying some minimal properties.
While we analyze the complexity of $\pt$ in terms of circuit size, we could similarly analyze the time or space complexity needed to evaluate $\pt$ under a uniform distinguishing model.

Recall, we assume that each $i \in \X$ can be represented by a boolean vector in $\set{0,1}^d$.
When relevant, we parameterize the class of distinguishers by the dimension $d$.
We assume that the distinguisher class $\A$ can be implemented by bounded fan-in circuits of size $s(d)$.
Typically, we think of $d \in \N$ as fixed by the prediction problem at hand; correspondingly, we can think of $s(d) = s$ to be a fixed bound on the circuit size for some $s \in \N$.

Finally, we make an observation and a few technical assumptions for convenience of analysis.
First, we note that Algorithm~\ref{alg:construction} can be implemented to build $\pt$ using computations up to a fixed precision of $\Theta(\eps)$, while maintaining the iteration complexity $T = O(1/\eps^2)$ (i.e.,\ with only a constant factor increase).
Assuming a fixed precision, we make some additional assumptions about the complexity of implementing logic over values from $[0,1]$, needed to implement the circuit built by Algorithm~\ref{alg:construction}.
Specifically, the circuit we build requires operations to:
\begin{itemize}
    \item fixed-precision addition over $[0,1]$;
    \item multiplication of fixed-precision values in $[0,1]$ by $b \in \set{0,1}$; and
    \item maintaining fixed-precision values in $[0,1]$ (i.e.,\ projection)
\end{itemize}
We assume that each of these operations can be implemented for $\Theta(\eps)$-precision in $w(\eps)$ gates.
Throughout, for technical simplicity, we assume that $s(d) \ge d \ge 3 \cdot w(\eps)$.

\subsection{Recovering bounds for \ONE and \TWO}

To begin, we show that Algorithm~\ref{alg:construction} produces circuits for \one and \two of complexity that scales linearly in the complexity of the distinguishers and inverse polynomially in the approximation parameter $\eps$.
The analysis will be identical for each notion, so we state the theorem for the stronger notion of \two.
\begin{theorem}
Let $\X \subseteq \set{0,1}^d$ for $d \le s \in \N$, and suppose $\A$ is a class of distinguishers implemented by size-$s$ circuits.
For any choice of \nature's distribution $\Dnat$ and $\eps > 0$, there exists an $(\A,\eps)$-\two predictor $\pt:\X \to [0,1]$ implemented by a circuit of size $O\left(s/\eps^2\right)$.
\end{theorem}
\begin{proof}
Initially, $p^{(0)}$ is a constant function, which Algorithm~\ref{alg:construction} knows and may access.
In each iteration of Algorithm~\ref{alg:construction}, $p^{(t+1)}$ is formed by combining a circuit for $p^{(t)}$ with a circuit computing $\delta^{(t)}_A$.
Computing $\delta^{(t)}_A$ requires two calls to $A$.
For \two distinguishers, we must pass $p^{(t)}_i$ as input to each copy of the circuit computing $A$.\footnote{For \one, in each iteration, evaluating $\delta^{(t)}_A$ requires two calls to some $A \in \A$, but no calls to $p^{(t)}$.
This property allows the circuit computing $\pt$ to make calls to each $\delta^{(t)}_A$ in parallel, so the depth of the eventual circuit can be bounded by the depth of the circuits computing $\A$ (with small overhead to compute the addition of all of the updates).}
Leveraging the existing circuit for $p^{(t)}$, we can feed its output wires into the corresponding inputs to $A$.
Thus to compute $p^{(t+1)}$, we require $2d$ wires from the output of $p^{(t)}$ into the distinguishers to compute $\delta^{(t)}_A$, which can be computed in $2s$ gates.
Finally, we need to perform a scalar multiplication, a finite-precision addition, and projection onto $[0,1]$, which by assumption can be handled in $3w \le d$ gates.
Thus, in total, each iteration adds $O(d + s) = O(s)$ new gates.
By the bounded iteration complexity of $T=O(1/\eps^2)$, the eventual circuit computing $\pt$ can be bounded in size by $O(s/\eps^2)$.
\end{proof}

\subsection{Obtaining \THREE}

In \three, we allow distinguishers to make oracle calls to the predictor in question $\pt$.
We model these distinguishers by circuits that may include $\pt$-oracle-gates.
We assume that such gates have $d$ labeled input wires, which can be resolved to an $i \in \X$, and output $\pt_i$.
We measure the complexity of such oracle circuits in terms of the size of the circuits and the number of oracle queries.
Again, we assume the dimension $d \in \N$ is fixed, so drop the dependence on the $d$.
\begin{definition}[Oracle circuits]
For $s,q \in \N$, a family of distinguishers $\A$ is implemented by $(s,q)$-oracle circuits if for each $A \in \A$, given a predictor $p:\X \to [0,1]$, the distinguisher $A^p$ is implemented by some size-$s$ circuit that includes at most $q$ $p$-oracle gates.
\end{definition}
Importantly, because Algorithm~\ref{alg:construction} builds up a circuit for $\pt$ iteratively, in each iteration $t \le O(1/\eps^2)$ the algorithm has access to a circuit computing $p^{(t)}$.
Thus, in order to implement each oracle distinguisher $A^{p^{(t)}}$, we can hardwire copies of $p^{(t)}$ into $A$ whenever it makes an oracle call.
Overall, we can prove the following bound on the total size of the circuit that computes $\pt$.

\begin{theorem}
Let $\X \subseteq \set{0,1}^d$ and $d \le s,q \in \N$, and suppose $\A$ is a class of distinguishers implemented by $(s,q)$-oracle circuits.
For any choice of \nature's distribution $\Dnat$ and $\eps > 0$, there exists an $(\A,\eps)$-\three predictor $\pt:\X \to [0,1]$ implemented by a circuit of size $s \cdot q^{O(1/\eps^2)}$.
\end{theorem}

\begin{proof}
For any class of $(s,q)$-oracle distinguishers, we show how to construct $\pt$ as in Algorithm~\ref{alg:construction} (without oracle gates).
This implementation leads to a recurrence relation for $s^{(t)}$ expressing the size of the circuit implementing $p^{(t)}$ for each of the $t$ iterations.
Solving the recurrence for $T = O(1/\eps^2)$ yields the stated bound.

To begin, the initial predictor $p^{(0)}$, the constant $1/2$ function, requires a circuit of size $s^{(0)} \le d \le s$.
Additionally, we note that the predictor $p^{(0)}$ is available to Algorithm~\ref{alg:construction} in iteration $t = 0$.
These two observations form the basis of our recurrence.

Consider iteration $t+1$ of Algorithm~\ref{alg:construction}.
By assumption each distinguisher is implemented by an $(s,q)$-oracle circuit with $p^{(t)}$-oracle circuits.
Importantly, the algorithm has access to a size $s^{(t)}$ circuit implementing $p^{(t)}$.
Thus, to evaluate $A^{p^{(t)}}$ for any $A \in \A$, we can implement oracle gates using the explicit circuit for $p^{(t)}$ in a total size of $s + q \cdot s^{(t)}$.
If the predictor is updated (i.e.,\ does not terminate) the new predictor $p^{(t+1)}$ requires two calls to the distinguishing circuit $A^{p^{(t)}}$, one call to $p^{(t)}$, a multiplication of $\delta^{(t)}_A$ by a scalar, an addition, and a projection onto $[0,1]$.
By our assumption that the necessary operations can each be computed in $w$ gates for $3w \le s$,
we can express the size of the updated circuit as follows.
\begin{align*}
    s^{(t+1)} &\le 2\cdot\left(s + q \cdot s^{(t)}\right) + s^{(t)} + 3w\\
    &\le 3s + (2q+1) \cdot s^{(t)}
\end{align*}
Solving the recurrence\footnote{The solution follows by the Taylor approximation $\frac{1}{1-x} = 1 + x + x^2 + \hdots$ taking $x = 2q + 1$.}, we bound the size of the circuit after $T = O(1/\eps^2)$ iterations needed to compute the $(\A,\eps)$-\three predictor $\pt$.
\begin{align*}
    s^{(T)} &\le 3s \cdot \frac{(2q+1)^{T+1} - 1}{2q}\\
    &\le s \cdot q^{O(T)}
\end{align*}
By the bound on the number of iterations from Lemma~\ref{lem:iterations}, the bound of $s \cdot q^{O(1/\eps^2)}$ on the circuit complexity of $\pt$ follows.
\end{proof}

\subsection{Obtaining \FOUR}

In all the previous levels of OI, we worked without loss of generality with a distinguisher class of fixed concrete complexity (e.g.,\ size-$s$ circuits).
Because \four considers distinguishers that take the description of $\pt$ as input, it is most natural to consider a family whose complexity scales with that of $\pt$.
We will consider a distinguishers implemented by a non-uniform family that is parameterized by both the dimension $d$ and the length of the description of $\pt$, which we denote $n$.
Concretely, we measure the complexity of the distinguishers and the resulting $\pt$ in terms of the size $s(d,n)$ of the circuits in $\A$.

For computations over the description of circuits to be meaningful, it is important to fix a representation of the description of $\pt$; we denote this canonical description by $\ptstring$.
The proposition follows by using the adjacency list representation of the circuit and applying any constant-stretch pairing function.
\begin{proposition}\label{prop:description}
Suppose $k,s \in \N$ and $k \le s$.
For any size-$s$ bounded fan-in circuit $c:\set{0,1}^k \to \set{0,1}$, there exists a canonical encoding $\langle c \rangle$ of the description of the circuit, and a canonical encoding $\langle \langle c\rangle, x \rangle$ of the description-input pairing, each of length $O(s \log(s))$ bits.
\end{proposition}

Throughout our discussion of \four, we will still assume that the dimension $d$ is fixed; thus, we will assume that the distinguishers are of fixed size $s(d)$ with respect to the dimension, but growing with respect to the length of the encoding $n$,
\begin{gather*}
    s(d,n) = s(d) + f(n)
\end{gather*}
for various functions $f$.
As before, we will assume that the base distinguisher size $s$ is at least on the order of the dimension $d$ and cost of operations $w$, so that $d + w \le O(s(d))$.
Given this assumption, and the simplicity of the initial constant predictor $p^{(0)}$, for all classes of distinguishers we consider here, we may assume a base size $s^{(1)}$ of the predictor $p^{(1)}$ after one iteration of Algorithm~\ref{alg:construction} as
\begin{gather*}
    s^{(1)} \le O(s(d)).
\end{gather*}
In addition to this base case, we can give a generic recurrence for the growth of $p^{(t)}$ in terms of the size of the distinguishers.
\begin{lemma}\label{lem:generic:four}
Suppose $\A$ is a class of \four distinguishers implemented by $s(d,n)$-sized circuits, and let $n(s) = O(s \log(s))$ upper bound the length of the encoding of size-$s$ circuits.
For any iteration $t$ of Algorithm~\ref{alg:construction}, let $p^{(t)}$ be the current predictor and $s^{(t)}$ denote the size of its circuit.
Then, the predictor size can be recursively bounded as
\begin{gather*}
    s^{(t+1)} \le O\left(s\left(d,n(s^{(t)})\right) \right) + s^{(t)}.
\end{gather*}
\end{lemma}
\begin{proof}
As in the analysis for \three, Algorithm~\ref{alg:construction} updates $p^{(t)}$ to $p^{(t+1)}$ using two calls to the algorithm that distinguishes $p^{(t)}$ and $\ps$, one call $p^{(t)}$, and three operations of cost at most $w$.
Thus, we can bound the growth rate of $s^{(t)}$ recursively as
\begin{align*}
    s^{(t+1)} &\le 2\cdot s\left(d,n(s^{(t)})\right) + s^{(t)} + 3w\\
    &\le O\left(s\left(d,n(s^{(t)})\right) \right) + s^{(t)},
\end{align*}
where the simplified inequality follows by the assumption that $w \le O(s(d)) \le O(s(d,n))$.
\end{proof}
We can apply Lemma~\ref{lem:generic:four} to obtain bounds on $s^{(T)}$, under various concrete assumptions about $s(d,n)$.
In particular, for different choices of $s(d,n)$, the different levels of OI can be implemented in the framework of \four.
For each level, we recover the bounds on the resulting circuit size up to low order terms (e.g.,\ polylogarithmic factors in the distinguisher size).
This blow-up follows as a consequence of the generality of the framework of \four, and the need to encode $\ptstring$ as input to the distinguisher and decode $\ptstring$ if the distinguisher wants to evaluate $\pt$.

\paragraph{Sublinear distinguishers.}
To start, suppose we choose a class of distinguishers of complexity independent of $\pt$, $s(d,n) = s(d)$.
In this case, as $\pt$ grows in complexity, the distinguishers' access to $\ptstring$ becomes relatively limited.
Such distinguishers can still implement \one tests by ignoring $\ptstring$; we show that the complexity of such predictors grows similarly to the complexity needed to satisfy \one.
In fact, we bound the complexity of a \four predictor even for distinguishers that have (strongly) sublinear access to $\pt$.

\begin{theorem}
Let $\X \subseteq \set{0,1}^d$ for $d \in \N$.
For any fixed constant $\delta \in [0,1)$,
suppose $\A$ is a class of distinguishers implemented by size-$s(d,n)$ circuits for
\begin{gather*}
    s(d,n) \le s(d) + \left(\frac{n}{\log(n)}\right)^\delta.
\end{gather*}
For any choice of \nature's distribution $\Dnat$ and $\eps > 0$, there exists an $(\A,\eps)$-\four predictor $\pt:\X \to [0,1]$ implemented by a circuit of size $O\left(s(d)\cdot\eps^{-\frac{2}{1-\delta}}\right)$.
\end{theorem}
\begin{proof}
We bound the size of the predictor $\pt$ produced by Algorithm~\ref{alg:construction} when using a class of distinguishers with sublinear access to $\ptstring$.
Applying Lemma~\ref{lem:generic:four}, we can derive a recurrence for the growth rate of the predictor at each iteration,
\begin{align*}
    s^{(t+1)} &\le O\left(s(d) + \left(\frac{n(s^{(t)})}{\log(n(s^{(t)}))}\right)^\delta \right) + s^{(t)}\\
    &\le O(s(d)) + s^{(t)} + b\cdot \left(s^{(t)}\right)^\delta\addtag\label{four:sublinear:encoding}
\end{align*}
for some constant $b > 0$, where (\ref{four:sublinear:encoding}) follows by Proposition~\ref{prop:description} bounding the length of the description $\langle p^{(t)} \rangle$ in terms of $s^{(t)}$.
For a sufficiently large constant $B > 0$, this recurrence can be bounded by $s^{(T)} \le B \cdot s(d) \cdot T^{\frac{1}{1-\delta}}$.
First, we know $s^{(1)} \le O(s(d))$.
Then, inductively, assume that the claimed recurrence holds for all $t' \le t$.
\begin{align*}
s^{(t+1)} &\le O(s(d)) + B \cdot s(d) \cdot t^{\frac{1}{1-\delta}} + b \cdot \left(B \cdot s(d) \cdot t^{\frac{1}{1-\delta}}\right)^{\delta}\\
&\le O(s(d)) + B \cdot s(d) \cdot t^{\frac{1}{1-\delta}} + b\cdot B^\delta \cdot s(d)^\delta \cdot t^{\frac{\delta}{1-\delta}}\\
&\le B \cdot s(d) \cdot \left(t^{\frac{1}{1-\delta}} + t^{\frac{1}{1-\delta} - 1}\right)\\
&\le B \cdot s(d) \cdot (t+1)^{\frac{1}{1-\delta}}
\end{align*}
Recalling that Algorithm~\ref{alg:construction} terminates after $T= O(1/\eps^2)$ iterations, the resulting circuit for $\pt$ is bounded by $s^{(T)} \le O\left(s(d)\cdot\eps^{-\frac{2}{1-\delta}}\right)$ for fixed constant $\delta$.
\end{proof}

\paragraph{Quasilinear distinguishers.}
Next, we handle distinguishers whose complexity may grow quasilinearly in the size of the description $\ptstring$.
Our focus on quasilinear distinguishers is due to the fact that the circuit evaluation problem has quasilinear circuits; thus, using such distinguishers we can hope to simulate \two and \three.
Specifically, we rely upon the existence of universal circuits of the following form.
\begin{theorem}[Corollary of Theorem 3.1 from \cite{lipton2012amplifying}]
There exists a (polylog-time uniform) universal circuit family $\set{C_n}$ of size $O\left(n \cdot \poly\log(n)\right)$, such that for any bounded fan-in circuit $c:\set{0,1}^k \to \set{0,1}$ and input $x \in \set{0,1}^k$,
where $\card{\left\langle\langle c \rangle, x\right\rangle} = n$,
\begin{gather*}
    C_n\left(\left\langle\langle c \rangle, x\right\rangle\right) = c(x).
\end{gather*}
\end{theorem}
By pairing this result with Proposition~\ref{prop:description}, we can bound the size needed to evaluate $\pt$.
Recall that for a circuit $p$ of size s, the description $\card{\langle p \rangle}$ can be written in $O(s \log(s))$ bits.
Thus, in combination, we will assume that given the description $\langle p \rangle$, the circuit $p$ can be evaluated by a distinguisher using $s \cdot \left(\log(s)\right)^c$ gates for some $c \in \N$.

We analyze the growth rate of the circuit produced by Algorithm~\ref{alg:construction} if we allow distinguishers that can make $q$ calls to the current model predictor.
Specifically, we consider distinguishers of size
\begin{gather*}
    s(d,n(s)) \le s(d) + s \cdot \left(\log(s)\right)^c
\end{gather*}
This growth rate bounds the size of the predictors needed to implement \three within the framework of \four.
As we'll see, the dependence on $q$ and the base distinguisher size $s(d)$ remains roughly the same, with overhead due to encoding and decoding $p^{(t)}$ at each iteration, rather than wiring it directly.

\begin{theorem}
Let $\X \subseteq \set{0,1}^d$ for $d \in \N$.
For a constant $c > 0$,
suppose $\A$ is a class of distinguishers implemented by size-$s(d,n)$ circuits for
\begin{gather*}
    s(d,n(s)) \le s(d) + s \cdot \left(\log(s)\right)^c.
\end{gather*}
For any choice of \nature's distribution $\Dnat$ and $\eps \in (0,1)$, there exists an $(\A,\eps)$-\four predictor $\pt:\X \to [0,1]$ implemented by a circuit of size $s(d)\cdot q^{O(1/\eps^2)}\cdot(1/\eps)^{O(1/\eps^2)} \cdot \left(\log(s(d)\cdot q)\right)^{O(1/\eps^2)}$.
\end{theorem}
\begin{proof}
By Lemma~\ref{lem:generic:four}, we can bound the growth rate at the $t$th iteration as follows
\begin{align*}
s^{(t+1)} &\le O\left(s(d) + q \cdot s^{(t)}\cdot\left(\log(s^{(t)})\right)^c\right) + s^{(t)}\\
&\le b \cdot \left(s(d) + q\cdot s^{(t)}\cdot\left(\log(s^{(t)})\right)^c\right)
\end{align*}
for some constant $b > 0$.
For notational convenience, let $\Xi = s(d) \cdot (2c)^{2c}  b q$.
We will show that the recurrence can be bounded as follows.
\begin{align*}
    s^{(T)} &\le s(d) \cdot ((2c)^{2c}bq)^T \cdot \left((T-1)!\right)^{2c}\cdot \left(\log(s(d) \cdot (2c)^{2c}bq)\right)^{cT}\\
    &= s(d) \cdot ((2c)^{2c}bq)^T \cdot \left((T-1)!\right)^{2c}\cdot \left(\log(\Xi)\right)^{cT}
\end{align*}
Again, we know that $s^{(1)} \le O(s(d))$.
To see the recurrence, consider the expression for $s^{(t+1)}$ assuming the inductive hypothesis for $s^{(t)}$.
\begin{align*}
s^{(t+1)} &\le b \cdot s(d) + bq\cdot s^{(t)}\cdot\left(\log(s^{(t)})\right)^c\\
&\le  b \cdot s(d) + bq \cdot \left(s(d) \cdot \left((2c)^{2c}bq \right)^t \cdot \left((t-1)!\right)^{2c}\cdot \left(\log(\Xi)\right)^{ct}\right)\\
&\qquad\qquad\qquad\qquad\qquad\qquad\cdot \left(\log\left(s(d) \cdot \left((2c)^{2c} bq)\right)^t \cdot \left((t-1)!\right)^{2c}\cdot \left(\log(\Xi)\right)^{ct}\right)\right)^c\\
&\le s(d) \cdot (2c)^{2ct} \cdot (bq)^{t+1}\cdot \left((t-1)!\right)^{2c} \cdot \left(\log(\Xi)\right)^{ct}\\
&\qquad\qquad\qquad\qquad\qquad\qquad\cdot
\left(t \cdot \log(\Xi) + 2ct\cdot \log(t) + ct\cdot\log\log(\Xi)\right)^c\\
&\le s(d) \cdot (2c)^{2ct} \cdot (bq)^{t+1}\cdot \left((t-1)!\right)^{2c} \cdot \left(\log(\Xi)\right)^{ct} \cdot (2ct)^{2c}\cdot \log(\Xi)^c\\
&= s(d) \cdot \left((2c)^{2c}bq\right)^{t+1} \cdot \left(t!\right)^{2c} \cdot \log\left(\Xi\right)^{c\cdot (t+1)}
\end{align*}
Taking $b$ and $c$ to be constants,
the bound on the recurrence can be simplified and bounded as follows,
\begin{gather*}
s^{(T)} \le s(d) \cdot q^{O(1/\eps^2)} \cdot \left(1/\eps\right)^{O(1/\eps^2)} \cdot \left(\log(s(d)\cdot q)\right)^{O(1/\eps^2)}
\end{gather*}
where $T \le O(1/\eps^2)$ follows by the iteration complexity of Algorithm~\ref{alg:construction}.
\end{proof}

\paragraph{Polynomial distinguishers.}
In the limits of efficient distinguishers, we may allow the computation over $\ptstring$ to grow polynomially.
Such distinguishers, in principle, could run nontrivial tests on the description $\ptstring$ itself beyond simply evaluating $\pt$.
We show that predictors satisfying this strong notion of OI exist, in complexity independent of that of $\ps$, albeit doubly exponential in the number of iterations of Algorithm~\ref{alg:construction}.

\begin{theorem}
Let $\X \subseteq \set{0,1}^d$ for $d \in \N$.
For any constant $k \in \N$, suppose $\A$ is a class of distinguishers implemented by size-$s(d,n)$ circuits for
\begin{gather*}
    s(d,n) \le \left(s(d) + n\right)^k.
\end{gather*}
For any choice of \nature's distribution $\Dnat$ and $\eps > 0$, there exists an $(\A,\eps)$-\four predictor $\pt:\X \to [0,1]$ implemented by a circuit of size $O\left(s(d)^{2^{O(1/\eps^2)}}\right)$.
\end{theorem}
\begin{proof}
Applying Lemma~\ref{lem:generic:four}, we can bound the growth rate of the size of $p^{(t)}$ for all $t \ge 1$ loosely through the following recurrence.
\begin{align*}
    s^{(t+1)} &\le \left(s(d) + n(s^{(t)})\right)^k + s^{(t)}\\
    &\le \left(3 \cdot s^{(t)}\log(s^{(t)})\right)^k\\
    &\le \left(3 \cdot s^{(t)}\right)^{2k}
\end{align*}
This recurrence can be bounded as
\begin{gather*}
s^{(T)} \le 3^{((2k)^{T+1}-1)/(2k-1)} \cdot \left(s(d)\right)^{(2k)^T}    
\end{gather*}
for a sufficiently large constant $B$.
To start, $s^{(1)} \le O(s(d))$.
Then, inductively, we bound the growth as follows.
\begin{align*}
    s^{(t+1)} & \le 3^{2k}\cdot \left(3^{((2k)^{t+1}-1)/(2k-1)} \cdot\left(s(d)\right)^{(2k)^{t}}\right)^{2k}\\
    &\le 3^{((2k)^{t+2}-1)/(2k-1)} \cdot \left(s(d)\right)^{(2k)^{t+1}}
\end{align*}
Again, by the bound on the number of iterations $T = O(1/\eps^2)$ and by the assumption that $k = O(1)$ is a constant, we can simplify the expression for the size of $\pt$ to $s^{(T)} \le s(d)^{2^{O(1/\eps^2)}}$.
\end{proof}

\subsection{Learning \OI Predictors from Samples}
\label{sec:beyond:samples}

While we describe Algorithm~\ref{alg:construction} primarily to establish the complexity of representing predictors $\pt$ that satisfy \three and \four, in principle, it could be implemented as a learning algorithm that works over labeled samples $\set{(i,\os_i)}$ drawn from nature.
We remark on the statistical and computational complexity of learning such predictors.

\paragraph{Bounding the sample complexity.}
The main statistical cost for implementing Algorithm~\ref{alg:construction} comes from the need to estimate $\Delta_A^{(t)}$ for each $A \in \A$ at each iteration $t \le O(1/\eps^2)$.
Using standard concentration inequalities and union bounding, we can upper bound the number of samples per fixed iteration needed to estimate $\Pr_{\Dnatsample}\lr{A(i,\os_i;p^{(t)} = 1}$ to accuracy $\Theta(\eps)$.
For instance, applying Hoeffding's inequality, we can see that $m_t \ge \Omega\left(\frac{\log(\card{\A}/\delta)}{\eps^2}\right)$ labeled samples suffice to estimate $\Delta^{(t)}_A$ for each $A \in \A$ with probability at least $1-\delta$.

Naively, we can take a fresh set of labeled samples for each iteration, so by iteration complexity $T \le O(1/\eps^2)$, then in total, $m_{1:T} \ge \Omega\left(\frac{\log(\card{\A}/\delta\eps)}{\eps^4}\right)$ labeled samples suffice.
Importantly, we cannot simply union bound over the iterations because the statistics we need to estimate at the $t$th iteration depend on $p^{(t)}$, and thus, depend on the estimates from earlier iterations.
More sophisticated approaches to bounding the sample complexity are possible.
For instance, to beat the approach of naive resampling, \cite{hkrr} apply the machinery for proving generalization in adaptive data analysis via differential privacy \cite{dmns,dwork2015generalization,dwork2015reusable,dwork2015preserving,bnsssu15,jung2020moment}, resulting in improved sample complexity.\footnote{\cite{kim} covers the approach of \cite{hkrr} for establishing generalization in detail.}

\paragraph{Bounding the computational complexity.}
While we bound the iteration complexity efficiently, naively, each iteration requires at least $\Omega(\card{\A})$ time to iterate through each $A \in \A$.
If we take $\A$ to be an arbitrary class of distinguishers, this complexity bound cannot be improved significantly under natural complexity-theoretic assumptions, even for \one.
In particular, in the context of learning multi-accurate and multi-calibrated predictors, \cite{hkrr} demonstrated that the search problem for a violated constraint requires agnostic learning, a notoriously hard problem (e.g., \cite{regev2010learning,kalai2008agnostic,feldman2010distribution}).

On the positive side, \cite{hkrr} also showed a reverse direction:  if $\C$ is taken to be a class of (agnostically) learnable functions, then $(\C,\alpha)$-multi-calibrated predictors are also learnable (in complexity that scales with the complexity of learning $\C$).
By analogy, if we take $\A$ to be a structured class of distinguishers there is hope that the problem of searching for a distinguisher $A \in \A$ with significant advantage may be reducible to an efficient learning problem.
Such a reduction would follow from techniques similar to those discussed in \cite{hkrr,kim}.

\section{Barriers to Efficient \OI}
\label{sec:lowerbound}

We define ensembles of scalably hard functions, and show that they imply the hardness of \three.  In Section \ref{subsec:cliques-OI} we show a candidate construction of scalably hard functions from clique counting and deduce that hardness of clique counting implies hardness of \three. In Section \ref{subsec:PSPACE-OI} we show results based on the existence of hard functions in $\PSPACE$. 

See Section \ref{sec:overview} for a technical overview of these results, and for a discussion of further candidates based on fine-grained complexity assumptions.

\subsection{Scalable Hardness and Hardness of OI}
\label{subsec:scalable}

\begin{definition}[Ensembles of functions, of distributions and of collections thereof]

An ensemble of functions $\{f_n : X_n \rightarrow Y_n\}_{n \in {\cal N}}$ is {\em computable} in time $t(n)$ if there exists a Turing machine that, on inputs of the form $\{(1^n,x) : x \in X_n\}$, runs in time $t(n)$ and correctly computes $f_n(x)$. Similarly, an ensemble of distributions $\{D_n\}_{n \in {\cal N}}$ is {\em sampleable} in time $t(n)$ if there exists a (probabilistic) Turing machine that, on input $1^n$, runs in time $t(n)$, and whose output distribution is $D_n$.

An ensemble of {\em collections} of functions ${\cal F} = \{ \{ f_{i,n} : X_n \rightarrow Y_n \}_{i=1}^{m(n)} \}_{n \in {\cal N}}$, where the $n$-th collection is of size  $m(n)$, is computable in time $t(n)$ if there exists a Turing machine that, on inputs of the form $\{(1^n,i,x) : i \in [1,\ldots,m(n)], x \in X_n\}$, runs in time $t(n)$ and outputs $f_{i,n}(x)$. We say that computing the ensemble {\em requires} randomized time $\ell(n)$ if it is not computable in randomized time $o(\ell(n))$ (i.e. there is no Turing machine running in this time that computes the function correctly w.h.p. over its coins). Similarly, an ensemble of {\em collections} of distributions ${\cal D} = \{ \{ D_{i,n} \}_{i=1}^{m(n)} \}_{n \in {\cal N}}$, is {\em sampleable} in time $t(n)$ if there exists a (probabilistic) Turing machine that, on inputs of the form $\{(1^n,i) : i \in [1,\ldots, m(n)]\}$, runs in time $t(n)$, and whose output distribution is $D_{i,n}$.

\end{definition}

\begin{definition}[Scalable hardness]
\label{def:scalable-funcs}

Let ${\cal F} =\{ \{f_{i,n}: X_{n} \rightarrow Y_n\}_{i=1}^{m(n)} \}$ be an ensemble of collections of functions, where $\{X_{n},Y_n \subseteq \{0,1\}^{\poly(n)}\}$ and $m(n) = \poly(n)$.
We say that the ensemble ${\cal F}$ has {\em scalable hardness} if the following conditions hold:

\begin{itemize}
    \item {\bf Complexity bounds.}  The ensemble ${\cal F}$ can be computed in time $t_f(n)$, and {\em requires} randomized time $\ell_f(n)$.
    
    \item {\bf Downwards self-reduction.} There is an oracle Turing Machine $Q$ s.t.:
    \begin{align*}
    \forall n \in {\cal N}, i \in [1,\ldots,m(n)], x \in X_{n}: Q^{f_{i-1,n}}(1^n,i,x) = f_{i,n}(x),
    \end{align*}
    where we define $f_{0,n}$ to be the all-zero function $\bar{0}$ (i.e. for $i=1$ the oracle does not help $Q$ in its computation).
    Let $t_Q(n)$ and $q_Q(n)$, which we refer to as the {\em runtime and query complexity (respectively) of ${\cal F}$'s downward self-reduction}, bound the (worst case) running time and query complexity (respectively) of $Q$ on inputs of the form $(1^n,\cdot,\cdot)$.
    
    \item {\bf Random self-reduction.} There is an oracle Turing Machine $R$, an error rate $\errate:{\cal N} \rightarrow [0,1]$, and an ensemble of collections of distributions ${\cal D} = \{ \{D_{i,n}\}_{ i \in [1,\ldots,m(n)]} \}_{n \in {\cal N}}$ that can be sampled in time $t_D(n)$, where each $D_{i,n}$ is over $X_{n}$, s.t. for every $ n \in {\cal N}, i \in [1,\ldots,m(n)]$ and every function $f^*: X_{n} \rightarrow \{0,1\}$, if it holds that $\Pr_{x \sim D_{i,n}}  \lr{ f^*(x) \neq f_{i,n}(x) }  < \errate(n)$,
then:  
    \begin{align*} \forall x \in X_{n}: \Pr \lr{ R^{f^*}(1^n,i,x) = f_{i,n}(x) } \geq 2/3, 
    \end{align*}
    where this latter probability is only over the coin tosses of $R$. 
    
    Let $t_R(n)$ and $q_R(n)$, which we refer to as the {\em runtime and query complexity of ${\cal F}$'s random self-reduction} (respectively),  bound the (worst case) running time and query complexity (respectively) of $R$ on inputs of the form $(1^n,\cdot,\cdot)$. We refer to $\errate(n)$ as  {\em ${\cal F}$'s error rate} and we refer to $t_D(n)$ as {\em ${\cal F}$'s sampling time}.
\end{itemize}

\end{definition}

\paragraph{Non-Boolean to Boolean scalable hardness.} We will be particularly interested in scalably hard {\em Boolean} functions, where the range is $\{0,1\}$ (for all $n \in {\cal N}$). We use the Goldreich-Levin hardcore predicate \cite{goldreich-levin} to transform non-Boolean scalably hard functions into Boolean ones by taking an inner product of the output with a random input string. The degradation in the parameters is polynomial in the output length of the original functions:

\begin{proposition}[Non-Boolean to Boolean scalable hardness]
\label{prop:booleanize}

Let ${\cal F} = \{ f_{i,n}: X_n \rightarrow \{0,1\}^{y(n)} \}$ be an ensemble of collections of functions with scalable hardness (see Definition \ref{def:scalable-funcs}) where each collection in the ensemble is of size $m(n)$ and the range of the functions is $\{0,1\}^{y(n)}$. Let $t_f(n)$ and $\ell_f(n)$  be ${\cal F}$'s complexity upper and lower bounds (respectively), let  $t_Q(n)$ and $q_Q(n)$ be ${\cal F}$'s downward self-reduction's runtime and query complexity (respectively), let  $t_R(n)$ and $q_R(n)$ be ${\cal F}$'s random self-reduction's runtime and query complexity (respectively), let $t_D(n)$ be ${\cal F}$'s sampling time and let $\errate(n)$ be ${\cal F}$'s error rate.

Then there exists an ensemble ${\cal G}= \{ g_{i,n}: \left( X_n \times \{0,1\}^{y(n)} \right)\rightarrow \{0,1\} \}$ of collections of {\em Boolean} functions with scalable hardness, where each collection is of the same size $m(n)$ and where the following hold:
\begin{itemize}
    \item  The complexity upper and lower bounds (respectively) of ${\cal G}$ are $t_g(n) = O(t_f(n))$ and $\ell_g(n) = \Omega(\ell_f(n) / y(n) \cdot \log(y(n)))$.
    
    \item  The runtime and query complexity of ${\cal G}$'s downwards self-reduction are (respectively) $t'_Q(n) = O(t_Q(n) + (q_Q(n) \cdot y^2(n)))$ and $q'_Q(n) = O(q_Q(n) \cdot y(n))$.
    
    \item  The runtime and query complexity of ${\cal G}$'s  random self-reduction are (respectively) $t'_R(n) =\tilde{O}(t_R(n) + (q_R(n) \cdot y^2(n)))$ and $q'_R(n) = \tilde{O}(q_R(n) \cdot y(n))$. The sampling time of ${\cal G}$ is $t'_D(n) = O(t_D(n) + y(n))$ and the error rate is $\errate'(n) = ((1/4 - \alpha) \cdot \errate(n))$ for an arbitrarily small constant $\alpha > 0$.
\end{itemize}
\end{proposition}

\begin{proof} The collection ${\cal G}$ is defined by $g_{i,n}(x,r) = \langle x , r \rangle$, where $x \in X_n, r \in \{0,1\}^{y(n)}$ and $\langle \cdot, \cdot \rangle$ is the inner product over $\GF[2]$. The complexity upper bound follows by construction. The lower bound follows because computing $f_{i,n}$ on an input $x$ can be reduced to computing each of its $y(n)$ output bits, and each output bit can be computed by a call to $g_{i,n}$ (where we take the majority of $\log(y(n))$ such calls to $g_{i,n}$ to reduce any probability of error).

The downwards self-reduction for $g_{i,n}$ on input $(x,r)$ uses the downwards self-reduction for $f_{i,n}$ on input $x$. Each call to $f_{i-1,n}$ can be simulated as described above, using $y(n)$ calls to $g_{i-1,n}$. Thus, the increase in the query complexity is a multiplicative $y(n)$ factor, the increase in the runtime is due to the time needed to make these additional queries. 

The random self-reduction for $g_{i,n}$ is to a distribution $D'$ obtained by sampling $x$ from ${\cal F}$'s ``hard'' distribution $D_{i,n}$ and appending a uniformly random Boolean string $r$ of length $y(n)$. The sampling time follows by construction. The random self-reduction $R'$ for $g_{i,n}$ gets access to an oracle $g^*$ that computes $g_{i,n}$ correctly with high probability over an input drawn from $D'$. On input $(x,r)$, $R'$ computes $f_{i,n}$ on $x$ and takes the inner product with $r$. Computing $f_{i,n}(x)$ is done by applying the random self-reduction for $f_{i,n}$ on input $x$. The latter reduction needs access to an oracle $f^*$ with error rate at most $\mu$ (over inputs drawn from $D_{i,n}$). We use the oracle to $g^*$ to simulate such an oracle. Let $\alpha > 0$ be a constant, and assume that for input $x \in X_n$:
\begin{align*}
    \Pr_{r \sim U_{y(n)}} \left[ g^*(x,r) = g_{i,n}(x,r) \right] \geq 3/4 + \alpha.
\end{align*}
In this case, we can reconstruct the $j$-th bit of $f_{i,n}(x)$ with probability greater than half by XORing the values $g^*(x,r)$ and $g^*(x,r \oplus e^{(j)})$ (where $r$ is uniformly random and $e^{(j)}$ is the $j$-th unit vector). By repeating this $O(\log(y(n)) + \log(q_R(n)))$ times and taking the majority, we reconstruct the $j$-th bit of $f_{i,n}(x)$ with all but $1/(100q_R(n) \cdot y(n))$ probability. Repeating this for each of the bits and taking a union bound, we recover $f_{i,n}(x)$ with all but $1/100q_R(n)$ probability.

Now, since the error rate of $g^*$ is at most $\mu(1/4 - \alpha)$, we get that:
\begin{align*}
    \Pr_{x \sim D_{i,n}} \left[ \Pr_{r \sim U_{y(n)}} \left[ g^*(x,r) \neq g_{i,n}(x,r) \right] > 1/4 - \alpha \right] < \mu.
\end{align*}
Thus, when $R'^{g^*}$ invokes $f_{i,n}$'s random self-reduction $R$, it is simulating an oracle $f^*$ that, for all but a $\mu$-fraction of the inputs drawn from $D_{i,n}$, computes $f_{i,n}$ correctly with all but $1/100q_R(n)$ probability over its own coins. By the guarantee of $f_{i,n}$'s random self-reduction, and taking a union bound over the error probability in answering all of $R$'s queries, we conclude that the invocation of $R$ will return the correct answer with probability at least $(2/3 - 1/100)$. The error probability can be reduced to $2/3$ by repeating (a constant number of times).

The number of queries made by the reduction $R'$ is (by construction) $O(q_R(n) \cdot y(n) \cdot (\log(y(n)) + \log(q_R(n))))$. The running time of $R'$ is $O(t_R(n) + (q_R(n) \cdot (\log(y^2(n)) + \log(q_R(n)))))$.
\end{proof}

\paragraph{Hardness of OI from scalable hardness.} We show that ensembles of scalably hard Boolean functions imply that OI is hard. Note that by Proposition \ref{prop:booleanize} we derive a similar conclusion for ensembles of scalably hard non-Boolean functions. The complexity upper and lower bounds follow by construction.

\begin{theorem}
\label{thm:scalable-to-OI-hardness}

Suppose ${\cal F}$ is an ensemble of collections of functions with scalable hardness (see Definition \ref{def:scalable-funcs}), where the functions are all Boolean (i.e. $Y_n= \{0,1\}$ for all $n$), and each collection in the ensemble is of size $m(n)$. Let $t_f(n)$ and $\ell_f(n)$  be ${\cal F}$'s complexity upper and lower bounds (respectively), let  $t_Q(n)$ and $q_Q(n)$ be ${\cal F}$'s downward self-reduction's runtime and query complexity (respectively), let  $t_R(n)$ and $q_R(n)$ be ${\cal F}$'s random self-reduction's runtime and query complexity (respectively), and let $t_D(n)$ be ${\cal F}$'s sampling time. Suppose further that ${\cal F}$'s error rate $\mu(n)$ is at most 0.06.

Then there exist:

\begin{itemize}

    \item an ensemble of predictors $\ps = \{\ps_n: Z_n \rightarrow \{0,1\}\}$ , where $Z_n \subseteq \left( [1,\ldots,m(n)] \times X_n \right)$, that is computable in time $O(t_f(n))$,
    
    \item an ensemble of distributions $H = \{H_n \}$, which can be sampled in time $O(t_D(n))$, where each $H_n$ is over $Z_n$,
    
    \item and an ensemble $\{{\A}_n\}$ of collections of distinguishers, where the $n$-th collection is of size $m(n)$, that is computable in time 
    $O \left( t_Q(n) + (t_R(n) \cdot q_Q(n) \cdot \log(q_Q(n))) \right)$,
    
\end{itemize}
    
such that for every ensemble of predictors $\{ \pt_n: Z_n \rightarrow [0,1] \}$ computable in time $o\left( \frac{\ell_f(n)}{q_R(n)} - t_R(n) \right)$, for infinitely many values of $n$, the predictor $\pt_n$ is {\bf not}
 $(\A_n,\varepsilon(n))$-OI with respect to $\ps_n$ and the distributions $H_n$, where $\varepsilon(n) = 1/100m(n)$.
    
\end{theorem}

\begin{proof}

Fix a value of $n$, and let $m = m(n)$. We construct $\ps = \ps_n$ and the distribution $H = H_n$ as follows:

\begin{itemize}
    \item The domain is $Z = Z_n = \{(i,x): i \in [1,\ldots, m(n)], x \in X_{i,n} \}$, where $X_{i,n}$ is the domain of the $i$-th function in the collection ${\cal F}_n$. We use $Z_i$ to denote the subset of $Z$ where the first coordinate is fixed to $i$ (note that these subsets are disjoint).
    
    \item For an input in $Z$ we have $\ps(i,x) = f_{i,n}(x)$, where $f_{i,n}$ is the $i$-th function in the collection ${\cal F}_n$. Note that the range of $\ps$ is Boolean.
    
    \item The distribution $H$ (over $Z$), has a uniform marginal distribution on $i \in [1,\ldots,m]$. Conditioning on a fixed value of $i$, the conditional distribution on $x$ is $D_i = D_{i,n}$ (the $i$-th distribution in the collection ${\cal D}_n$). \end{itemize}

The collection of distinguishers $\{A_1,\ldots,A_m\}$ is described in Figure \ref{alg:distinguisher}. Each of these distinguishers gets oracle access to a predictor $p: Z \rightarrow [0,1]$. On input $(j \in [1,\ldots, m],x \in X_{j,n},\oo \in \{0,1\})$, they try to distinguish whether or not $\oo$ is drawn by $\ps$. 

\begin{algorithm}[t!]
\caption{\label{alg:distinguisher} \textsc{Distinguisher $A_i^p(j,x,\oo)$}}
\hrulefill

if $i \neq j$, output 0

otherwise, if $i=1$, output 1 iff $\oo = Q^{\bar{0}}(1^n,1,x)$ \hfill\texttt{//using the all-0 oracle to compute $f_1$}

otherwise, if $i \geq 2$:
\begin{algorithmic}

\STATE for $x' \in X_{i-1,n}$, let:

\begin{itemize}
    \item $g(x') = \lceil p(i-1,x') \rfloor$ \hfill\texttt{//rounding $p$'s answers}
    
    \item $g'(x') = R^{g}(1^n,i-1,x')$ \hfill\texttt{//random self-reduction}
    
    \item $g''(x')$ computes the majority answer in $t = O(\log(q_Q(n)))$ \\ independent executions of $g'$ on  $x'$ \hfill\texttt{//boosting the success probability}
\end{itemize}

\STATE output 1 iff $\oo = Q^{g''}(1^n,i,x)$ \hfill\texttt{//downwards self-reduction to compute $f_{i,n}(x)$}
\end{algorithmic}
\hrulefill
\end{algorithm}

The bound on the running time of the $A_i$'s follows by construction: $A_1$ runs in time $t_Q(n)$, simulating $Q(1^n,1,x)$ and answering all oracle calls using the 0 oracle. For $i \geq 2$, the distinguisher $A_i$ runs the downwards self-reduction, answering each of its oracle queries (which are meant to be answered by $f_{i-1,n}$) by running the random self-reduction $R(1^n,i-1,x)$ multiple times, and taking the majority of its answers (the repetition lowers the error probability, see below). The oracle queries in each invocation of  $R$ (which are to inputs in $X_{i-1,n}$) are answered using the given predictor $p$ (restricted to inputs in $Z_{i-1}$, i.e. with the first coordinate fixed to $(i - 1)$), rounding its answers to 0 or 1. By construction, the running time is $O(t_Q(n) + (t_R(n) \cdot q_Q(n) \cdot \log(q_Q(n)) ))$.

The intuition behind the hardness of achieving OI for the class of distinguishers ${\cal A}$ (see also Section \ref{sec:overview}) is that if $\pt$ is OI and is ``correct on average'' in computing $f_{i-1}$ over random inputs, then in the distinguisher $A_i$ (which gets oracle access to $\pt$), the procedure $g''$, which uses ${\cal F}$'s random self reduction, computes $f_{i-1}$ correctly w.h.p. on every input in $X_{i-1}$. By the downwards self-reducability of ${\cal F}$, this implies that for every input $(i,x,\oo)$, w.h.p. over its coins the distinguisher $A_i$ (with oracle access to $\pt$) accepts if and only if $\oo = f_i(n)$. Since it is always true that $\oo = f_i(n)$ when $\oo$ is drawn by $\ps$, the fact that $\pt$ is OI implies that it should also be true w.h.p. when $\oo$ is drawn by $\pt$. When we say a (real-valued) predictor $\pt$ ``correctly computes'' a Boolean function like $f_i$, we mean that when the value outputted by the predictor is {\em rounded}, the resulting boolean value agrees with $f_i$.

The distinguisher $A_1$ is constructed so that output indistinguishability w.r.t. this distinguisher implies that $\pt$ much be ``correct on average'' in computing $f_1$. This is the basis for an inductive argument, which shows that $pt$  must also be  ``correct on average'' in computing $f_i$ for every possible value of $i$. Plugging (the rounding of) $\pt$ into the random self-reduction $R$, we obtain a machine that computes ${\cal F}$ correctly for every input (w.h.p. over its coins). Since computing ${\cal F}$ requires randomized time $\ell_f(n)$, we conclude that computing $\pt$ requires time $\left( (\ell_f(n) / q_R(n)) - t_R(n) \right)$. The formal induction argument follows from the following two claims:

\begin{claim}[induction basis] \label{claim:lowerbound_induction_basis}
For every predictor $\pt:Z \rightarrow [0,1]$ that satisfies $(\{A_1\},\eps(n))$-OI w.r.t $\ps$, it must be the case that:
\begin{align*}
    \Pr_{x \sim D_{1}} \lr{\lceil \pt(1,x) \rfloor = f_{1,n}(x) } \geq 0.98.
\end{align*}
\end{claim}

\begin{proof}[Proof of Claim \ref{claim:lowerbound_induction_basis}]
The distinguisher $A_1$ accepts its input $(j,x,\oo)$ iff $j=1$ and $\oo = f_1(x)$. Since $p^*(j,x) = f_j(x)$ we have that:
\begin{align*}
    \Pr_{(j,x) \sim H, \os \sim \ps(j,x)} \lr{A_1(j,x,\os)=1} = 1/m(n).
\end{align*}
If $\pt$ is  $(\{A_1\},\eps(n))$-OI w.r.t $p^*$, where $\eps(n) = 1/100m(n)$, we have that:
\begin{align*}
    \Pr_{(j,x) \sim H, \ot \sim \pt(j,x)} \lr{A_1(j,x,\ot)=1} \geq 0.99/m(n).
\end{align*}
Now, since $A_1$ rejects when either $j\neq 1$ or $\ot \neq f_{1,n}(x)$, and since the distribution $H$ is uniformly random over its first coordinate, we conclude that:
\begin{align*}
    \Pr_{x \sim D_1, \ot \sim \pt(1,x)} \lr{\ot \neq f_{1,n}(x)} < 0.01.
\end{align*}
In particular, we conclude that:
\begin{align*}
    \Pr_{x \sim D_1} \lr{ \left| \pt(1,x) - f_{1,n}(x) \right| \geq 0.5 } < 0.02,
\end{align*}
and the claim follows.
\end{proof}

\begin{claim}[induction step] \label{claim:lowerbound_induction}
For every predictor $\pt:Z \rightarrow [0,1]$ and $i \in [2,\ldots,m]$. If $\pt$ is $(\{A_i\},\eps(n))$-OI w.r.t $\ps$ and also:
\begin{align} \label{eqn:induction_condition}
    \Pr_{x\sim D_{i-1}} \lr{\lceil \pt(i-1,x) \rfloor = f_{i-1,n}(x) } \geq 0.94,
\end{align}
then:
\begin{align*}
    \Pr_{x \sim D_{i}} \lr{\lceil \pt(i,x) \rfloor = f_{i,n}(x) } \geq 0.94,
\end{align*}
\end{claim}

\begin{proof}[Proof of Claim \ref{claim:lowerbound_induction}]

The condition in Equation \eqref{eqn:induction_condition} implies that $\pt$ is ``correct on-average'' for computing $f_{i-1}$: after rounding, it agrees with $f_{i-1}$ w.h.p. over random inputs from $D_{i-1}$. When we run the distinguisher $A_i$ with oracle access to $\pt$, by the random self-reducability of ${\cal F}$, we conclude that for each input $x \in X_{i-1}$, the procedure $g'$ computes $f_{i-1}$ correctly with probability $2/3$. The procedure $g''$, which runs $t = O(\log(q_Q(n)))$ independent invocations of $g'$, thus has an error probability bounded by $1/100q_Q(n)$. The procedure $Q^{g''}(1^n,i,\cdot)$ runs ${\cal F}$'s downwards self-reduction using $g''$ as an oracle. Taking a union bound over $Q$'s oracle calls, we conclude that the probability that they are all answered correctly is at least $0.99$. When this is the case, $Q^{g''}(1^n,i,\cdot)$ correctly computes the function $f_i$. Thus, we conclude that for any $x \in X_i$, the distinguisher $A_i(i,x,\oo)$ accepts w.p. at least 0.99 if $\oo=f_i(x)$, and otherwise it rejects w.p. at least 0.99.:
\begin{align*}
    \Pr_{x \sim D_i, \os \sim \ps(i,x)} \lr{A_i(i,x,\os)=1} \geq 0.99.
\end{align*}
Since $\pt$ is  $(\{A_i\},\eps(n))$-OI w.r.t $p^*$, where $\eps(n) = 1/100m(n)$, and since $A_i(j,x,\oo)$ always rejects if $j \neq i$ and the distribution $H$ is uniformly random over its first coordinate, it must be the case that:
\begin{align*}
    \Pr_{x \sim D_i, \ot \sim \pt(i,x)} \lr{A_i(i,x,\ot)=1} \geq 0.98.
\end{align*}
Now, since $A_i(i,x,\oo)$ rejects w.p. at least 0.99 when $\oo \neq f_{i,n}(x)$, we conclude that:
\begin{align*}
    \Pr_{x \sim D_i, \ot \sim \pt(i,x)} \lr{\ot \neq f_{i,n}(x)} < 0.03.
\end{align*}
In particular, we conclude that:
\begin{align*}
    \Pr_{x \sim D_i} \lr{ \left| \pt(i,x) - f_{i,n}(x) \right| \geq 0.5 } < 0.06,
\end{align*}
and the claim follows.

\end{proof}

\end{proof}

\subsection{Hardness of OI from Clique Counting}
\label{subsec:cliques-OI}

In this section we use the clique-counting problem to construct an ensemble of functions with scalable hardness, and derive a hardness result for OI. We describe this result, and conclude the section with a discussion of related work on the complexity of clique counting.

\begin{theorem}[Scalable hardness from clique counting]
\label{thm:scalable-clique}

For $k \in {\cal N}$, let $t_{\cliquec}(k,n)$ be the upper bound for counting the number of cliques of size $k$ in an $n$-vertex graph. Suppose further that there is no $o(\ell_{\cliquec}(k,n))$-time randomized algorithm for counting $k$-cliques. 

Then for every $m \in {\cal N}$, there exists an ensemble ${\cal F}$ of functions with scalable hardness (see Definition \ref{def:scalable-funcs}), where:
\begin{itemize}
    \item each collection in the ensemble is of size $m$, the domain is of size $\log |X_n| = \tilde{O}(n^2)$, and the range is $[n^m]$,
    \item the ensemble can be computed in time $t_f(n) = t_{\cliquec,m}(n)$ and requires time $\ell_{f,m}(n) = \ell_{\cliquec,m}(n)$,
    \item the runtime of ${\cal F}$'s downwards self-reduction is $t_Q(n) = \tilde{O}(n^3)$ and its query complexity is $q_Q(n) = n$,
    \item the runtime of ${\cal F}$'s random self-reduction is $t_R(n) = \tilde{O}(n^2)$ and its query complexity is $q_R(n) = \poly(\log n )$. ${\cal F}$'s sampling time is $\tilde{O}(n^2)$ and the error rate is $1/4 - \eps$ for an arbitrarily small constant $\eps > 0$.
\end{itemize}
\end{theorem}

\begin{proof}

In a nutshell, we want $f_{i,n}$, given an (unweighted) $n$-vertex graph, to output the number of cliques of size $i$ in that graph. Note, however, that we need $f_{i,n}$ to be a Boolean function. We achieve this by taking a hardcore predicate of the clique counting function. In particular, we use the Goldreich-Levin hardcore predicate \cite{goldreich-levin}, taking an inner product of the count with a random input string. The construction follows.

Taking ${\cal G}_n$ to be the set of $n$-vertex graphs, we define $f_{i,n}: {\cal G}_n \rightarrow [1,\ldots, \binom{n}{i}]$ to be the number of cliques of size $i$ in a given graph (note that this number is at most $ \binom{n}{i}$). The upper and lower bounds on the running time of computing $f_{i,n}$ follow directly from the complexity of clique counting.

\paragraph{Downwards self-reduction.} We use the downwards self-reduction for clique counting, which uses the observation that for every vertex $v$ in a graph $G$, there is a bijection betweeen the $i$-cliques in $G$ that include the vertex $v$, and the $(i-1)$-cliques in the graph $G^{(v)}$ of $v$'s neighbors (the graph where we keep edges between $v$'s neighbors and erase all other edges). When we sum the numbers of $(i-1)$ cliques in the $n$ graphs $G^{(v)}$, each $i$-clique in $G$ is counted exactly $i$ times. Dividing the sum by $j$ gives the number of $j$-cliques in $G$:
\begin{align*}
    i \cdot f_{i,n}(G) = \sum_{v \in [n]} f_{i-1,n}(G^{(v)}),
\end{align*}
where $G^{(v)}$ is the graph induced by vertex $v$'s neighbors (we identify the set of vertices with $[n]$). The downwards self-reduction runs in time $\tilde{O}(n^3)$ and makes $n$ oracle queries.

\paragraph{Random self-reduction.} We use the random self-reduction for clique counting \cite{GoldreichR18}, the claimed bounds follow from their result:

\begin{theorem}[\cite{GoldreichR18} Theorem 1.1, worst-case to average-case reduction
for counting cliques]
\label{thm:wc2ac-cliques}
For any constant $t$,
there exists a distribution $D_{t,n}$ on $n$-vertex graphs
and a $\tilde{O}(n^2)$-time worst-case to average-case reduction
of counting $t$-cliques in $n$-vertex graphs to counting $t$-cliques
in graphs generated according to this distribution, such that
the reduction outputs the correct value with probability $2/3$
provided that the error rate {\em(of the average-case solver)}
is a constant smaller than one fourth.
Furthermore, the reduction makes $\poly(\log n)$ queries,
and the distribution $D_{t,n}$ can be sampled in $\tilde{O}(n^2)$-time
and is uniform on a set of $\exp(\tilde{\Omega}(n^2))$ graphs.
\end{theorem}

\end{proof}

We remark that recent works of Boix-Adsera, Brennan and  Bressler \cite{Boix-AdseraBB19}, Goldreich \cite{Goldreich20}, and Dalirrooyfard, Lincoln and Vassilevska Williams \cite{DalirrooyfardLW20} show other worst-case to average-case reductions for clique counting and other fine-grained complexity problems. However, the reductions in these works require sub-constant error rates, and thus it is not clear that our construction can be instantiated using these results.

\paragraph{Hardness of OI.} Putting together Theorem \ref{thm:scalable-to-OI-hardness} (which gives scalably hard functions with $\polylog$ output length) and Proposition \ref{prop:booleanize} (which shows how to go from non-Boolean to Boolean scalable hardness) together, we derive a Boolean scalably hard function collection (the degradataion in the parameters is poly-logarithmic). Plugging this collection into Theoream \ref{thm:scalable-to-OI-hardness}, we obtain $p^*$ for which OI is hard, based on the complexity of counting cliques in a graph. The resulting corollary is stated below.

\begin{corollary}
\label{cor:clique-to-OI-hardness}

For $k \in {\cal N}$, let $t_{\cliquec}(k,n)$ be an upper bound for the time complexity of counting the number of cliques of size $k$ in an $n$-vertex graph. Suppose further that there is no $o(\ell_{\cliquec}(k,n))$-time randomized algorithm for counting $k$-cliques. 

Then for every $m \in {\cal N}$, there exist:

\begin{itemize}

    \item an ensemble of predictors $\ps = \{\ps_n: Z_n \rightarrow \{0,1\}\}$ , where the domain consists of strings of length $\tilde{O}(n^2)$, and $\ps$ is computable in time $\tilde{O}(t_{\cliquec}(m,n))$,
    
    \item an ensemble of distributions $H = \{H_n \}$, which can be sampled in quasi-linear time (i.e. time $\tilde{O}(n^2)$), where each $H_n$ is over $Z_n$,
    
    \item and an ensemble $\{{\A}_n\}$ of collections of distinguishers, where each collection is of size $m$, that is computable in time 
    $\tilde{O}(n^3)$,
    
\end{itemize}
    
such that for every ensemble $\{ \pt_n: Z_n \rightarrow [0,1] \}$ of predictors that are computable in time $\left( \ell_{\cliquec}(m,n) \cdot \log^{-\omega(1)}(n) \right)$, for infinitely many values of $n$, the predictor $\pt_n$ is {\bf not}
 $(\A_n,\varepsilon(n))$-OI with respect to $\ps_n$ and the distributions $H_n$, where $\varepsilon(n) = 1/100m$.

\end{corollary}

\paragraph{The complexity of clique counting.} The complexity of clique counting, and of clique detection (which is no harder), has been widely studied and related to other computational problems. Chen {\em et al.}  \cite{ChenCFHJKX05} and Chen {\em et al.} \cite{ChenHKX06} show that detecting $k$-cliques in $n^{o(k)}$ time would refute the Exponential Time Hypothesis (ETH) \cite{ImpagliazzoP01}, and that counting $k$-cliques in $n^{o{(k)}}$-time would refute $\#$ETH. The clique detection problem is complete for the parameterized complexity class $W[1]$ (see Downey and Flum \cite{DowneyF99}), and the clique counting problem is complete for $\#W[1]$ (Flum and Grohe \cite{FlumG04}). It has been conjectured that the fastest known algorithms for clique counting and detection are tight: in particular, that there is no $n^{\omega k - o(1)}$-time algorithm for detecting $3k$-cliques (where $\omega$ is the matrix multiplication exponent). Abboud, Backurs and Vassilevska Williams \cite{AbboudBW18} showed several implications of this conjecture. 
For constant $k$, Abboud, Lewi and Williams \cite{AbboudLW14} show that the $k$-SUM conjecture implies a $n^{\lceil k/2 \rceil - o(1)}$ lower bound for $k$-clique detection.
Abboud \cite{Abboud19} shows that detecting $k$-cliques in $O(n^{\eps k/log k})$-time would refute the Set-Cover Conjecture. 

Under the conjecture (see above) that $3k-cliques$ cannot be counted in randomized time $n^{\omega k - o(1)}$, we obtain (for the same $k$) a distribution $\Dnat$ of nature that can be sampled in time $O(n^{k \omega})$, and a collection $\A$ of $3k$ distinguishers, which all run in $O(n^3)$ time, so that there is no $(\A,1/300k)$-\three predictor whose running time is smaller than $n^{\omega k}$ by more than a polylogarithmic factor.

\subsection{Scalable Hardness from $\BPP \neq \PSPACE$}
\label{subsec:PSPACE-OI}

In this section we construct an ensemble of functions with scalable hardness that are $\PSPACE$-hard to compute. This construction is based on the $\PSPACE$ complete problem of Trevisan and Vadhan \cite{TrevisanV07} (which is both random self-reducible and downwards self-reducible). Under the assumption that $\BPP \neq \PSPACE$, we derive a hardness result for OI. We note that in this result (unlike our hardness results that are based on fine-grained complexity), Natures predictor $\ps$ is not poly-time compuable (alternatively, Nature's distribution $\Dnat$ cannot be sampled in polynomial time). We also note the hard-to-obtain indistinguishability advantage is polynomially small (rather than a constant in the aforementioned results).

\begin{theorem}[Scalable hardness from $\BPP \neq \PSPACE$ \cite{TrevisanV07}, Lemma 4.1]
\label{thm:scalable-PSPACE}

Suppose that $\BPP \neq \PSPACE$. Then there exists an ensemble ${\cal F}$ of functions with scalable hardness (see Definition \ref{def:scalable-funcs}), where:
\begin{itemize}
    \item each collection in the ensemble is of size $m(n) = \poly(n)$, the domain and the range are of size $\exp(\poly(n))$,
    \item the ensemble can be computed in $\exp(\poly(n))$ time and requires super-polynomial time in $n$,
    \item the runtime and query complexity of ${\cal F}$'s downwards self-reduction are $\poly(n)$,
    \item the runtime and query complexity of ${\cal F}$'s random self-reduction are $\poly(n)$. ${\cal F}$'s sampling time is $\poly(n)$ and the error rate is is $1/4 - \eps$ for an arbitrarily small constant $\eps > 0$.
\end{itemize}
\end{theorem}

We note that Trevisan and Vadhan \cite{TrevisanV07} only claimed a random self-reduction with error rate $1/\poly(n)$. The error rate claimed above (and, indeed, error rates that are very close to 1) is implied by the worst-case to average-case (or rare-case) reduction of Sudan, Trevisan and Vadhan \cite{SudanTV01}.

Putting together Theorem \ref{thm:scalable-PSPACE} and Proposition \ref{prop:booleanize}, we derive a Boolean scalably hard function collection (the degradataion in the parameters is polynomial). Plugging this collection into Theoream \ref{thm:scalable-to-OI-hardness}, we obtain $p^*$ for which OI is hard, assuming $\BPP \neq \PSPACE$. The resulting corollary is stated below.

\begin{corollary}
\label{cor:PSPACE-to-OI-hardness}

Suppose that $\BPP \neq \PSPACE$. Then there exist:

\begin{itemize}

    \item an ensemble of predictors $\ps = \{\ps_n: Z_n \rightarrow \{0,1\}\}$ , where the domain consists of strings of length $\poly(n)$, and $\ps$ is computable in time $\exp(\poly(n))$,
    
    \item an ensemble of distributions $H = \{H_n \}$, which can be sampled in $\poly(n)$ time, where each $H_n$ is over $Z_n$,
    
    \item and an ensemble $\{{\A}_n\}$ of collections of distinguishers, where each collection is of size $m(n) = \poly(n)$, that is computable in time 
    $\poly(n)$,
    
\end{itemize}
    
such that for every ensemble $\{ \pt_n: Z_n \rightarrow [0,1] \}$ of predictors that are computable in time $\poly(n)$, for infinitely many values of $n$, the predictor $\pt_n$ is {\bf not}
 $(\A_n,\varepsilon(n))$-OI with respect to $\ps_n$ and the distributions $H_n$, where $\varepsilon(n) = 1/m(n)$.

\end{corollary}

\paragraph{Acknowledgments.}  MPK thanks Dylan McKay for suggesting useful references on the circuit evaluation problem. GNR thanks Amir Abboud for helpful discussions about the fine-grained complexity of clique counting.

\bibliographystyle{alpha}
\bibliography{refs}

\end{document}